\documentclass{article}

    \usepackage[workshop]{neurips_2021}

\usepackage[utf8]{inputenc} 
\usepackage[T1]{fontenc}    
\usepackage{hyperref}       
\usepackage{url}            
\usepackage{booktabs}       
\usepackage{amsfonts}       
\usepackage{nicefrac}       
\usepackage{microtype}      
\usepackage{xcolor}         

\usepackage{url}
\usepackage{hyperref}
\usepackage[small]{caption}
\usepackage{graphicx}
\usepackage{amsmath}
\usepackage{amsthm}
\usepackage{amssymb}
\usepackage[ruled,longend,linesnumbered]{algorithm2e}
\usepackage{subfigure}
\usepackage{wrapfig}

\usepackage{xcolor}
\newcount\comments  
\newcommand{\genComment}[2]{\ifnum\comments=1{\textcolor{#1}{\textsf{\footnotesize #2}}}\fi}

\usepackage{changes}

\newtheorem{lemma}{Lemma}
\newtheorem{theorem}{Theorem}

\title{MHER: Model-based Hindsight Experience Replay}

\author{%
  Rui Yang\thanks{Work was done during the internship with Tencent Robotics X.}  $\ ^1$ , Meng Fang$^{2,3}$, Lei Han$^2$, Yali Du$^4$, Feng Luo$^1$, Xiu Li$^1$\\
  $^1$Tsinghua University, $^2$Tencent Robotics X \\ $^3$Eindhoven University of Technology, $^4$King’s College London \\
  \texttt{yangrui19@mails.tsinghua.edu.cn, m.fang@tue.nl, lxhan@tencent.com } \\
  \texttt{yali.du@kcl.ac.uk, \{luof19@mails, li.xiu@sz\}.tsinghua.edu.cn}\\
}

\begin{document}

\maketitle

\begin{abstract}
Solving multi-goal reinforcement learning (RL) problems with sparse rewards is generally challenging. Existing approaches have utilized goal relabeling on collected experiences to alleviate issues raised from sparse rewards. However, these methods are still limited in efficiency and cannot make full use of experiences. In this paper, we propose \emph{Model-based Hindsight Experience Replay} (MHER), which exploits experiences more efficiently by leveraging environmental dynamics to generate virtual achieved goals. Replacing original goals with virtual goals generated from interaction with a trained dynamics model leads to a novel relabeling method, \emph{model-based relabeling} (MBR). Based on MBR, MHER performs both reinforcement learning and supervised learning for efficient policy improvement. Theoretically, we also prove the supervised part in MHER, i.e., goal-conditioned supervised learning with MBR data, optimizes a lower bound on the multi-goal RL objective. Experimental results in several point-based tasks and simulated robotics environments show that MHER achieves significantly higher sample efficiency than previous model-free and model-based multi-goal methods.

\end{abstract}

\section{Introduction}
Although reinforcement learning (RL) has been shown to be effective in a range of reward-driven problems \cite{mnih2015human,lillicrap2015continuous,haarnoja2018soft,DBLP:journals/corr/abs-2106-03050}, current RL algorithms require massive amounts of training data \cite{wang2016learning} and lack sample efficiency in sparse reward settings \cite{andrychowicz2017hindsight}.
In multi-goal RL, the problem of efficiency becomes more prominent as agents are required to accomplish multiple goals simultaneously. One of the most essential factors affecting sample efficiency in multi-goal RL is the sparse reward, in which case informative learning signals are very limited. Previous works have proposed many solutions such as reward shaping \cite{ng1999policy}, curriculum learning \cite{bengio2009curriculum}, exploration \cite{pathak2017curiosity,zhao2019maximum}, and hindsight relabeling \cite{DBLP:conf/ijcai/Kaelbling93,andrychowicz2017hindsight}. Among those solutions, learning from mistakes is a useful strategy to handle sparse rewards in multi-goal RL settings. Hindsight Experience Replay (HER) \cite{andrychowicz2017hindsight} remarkably improves sample efficiency through goal-relabeling that relabels failed experiences with actually achieved goals. Following HER, a few works are put forward to improve goal sampling methods \cite{fang2019curriculum,pitis2020maximum,DBLP:conf/nips/LiPA20,DBLP:conf/nips/EysenbachGLS20}, or utilize hindsight knowledge for supervised policy learning \cite{DBLP:conf/nips/SunLLZL19,ghosh2021learning} and adversarial imitation learning \cite{ding2019goal}.

Goal relabeling provides a practical paradigm for generating pseudo demonstrations to train control policies, and deep reinforcement learning algorithms further improve upon the efficiency of relabeling strategies~\cite{DBLP:conf/ijcai/Kaelbling93, andrychowicz2017hindsight}.
Most goal relabeling methods depend on trajectories and goals that an agent collects from environments. 
However, intelligent agents can achieve goals in complex environments even though they never encounter the exact same situation \cite{schaul2015universal}. This ability requires building representations of the dynamics from past experience that enable generalization to novel situations \cite{schaul2013better,singh2003learning}. Modeling dynamics offers an explicit way to represent an agent’s knowledge about the task. Existing methods \cite{nagabandi2018neural,DBLP:conf/nips/JannerFZL19,hafner2019dream,schrittwieser2020mastering} have apply neural models to greatly facilitate predicting physical dynamics and the consequences of actions, and provide a strong inductive bias for generalization to novel environment situations. With the dynamics model and the current policy, we can predict future states along with achieved goals. Can we exploit goals generated from model-based interaction for sample efficient multi-goal reinforcement learning?

In this paper, we propose a novel framework, \emph{Model-based Hindsight Experience Replay} (MHER), utilizing environmental dynamics to handle sparse rewards. In MHER, we introduce a new relabeling method, \emph{model-based relabeling}, and then minimize a joint loss based on the model-based relabeled data for efficient policy learning. Unlike previous hindsight relabeling methods that relabel transitions with goals achieved at a later point during the same trajectory, model-based relabeling leverages dynamics models to generate pseudo goals for guiding the learning of the policy, as shown in Figure \ref{fig:diagram}. The pseudo goals are collected in an efficient way without interaction with environments. Then goal-relabeling through imagined trajectories allows an agent to re-interpret its actions using a different goal from the perspective of the latest policy, leading to an implicit curriculum of goal relabeling guided by the policy. With the model-based relabeled data, we apply both supervised learning and reinforcement learning to update the policy. We theoretically prove that the policy can be improved by minimizing the goal-conditioned supervised learning loss \cite{ghosh2021learning} with the model-based relabeled data. To evaluate the performance of MHER, we conduct experiments on both point-based and Mujoco environments. Experimental results \footnote{\href{https://github.com/YangRui2015/Model-basedHER}{https://github.com/YangRui2015/Model-basedHER}.} show that MHER achieves significantly higher sample efficiency than previous works such as HER, Curriculum-guided HER \cite{fang2019curriculum}, Maximum Entropy-based Prioritization \cite{zhao2019maximum}, and Goal-Conditioned Supervised Learning \cite{ghosh2021learning}. 


The main contributions of this paper can be summarized as follows:
\begin{itemize}
    \item We present a new goal relabeling method, model-based relabeling, leveraging dynamics models to handle sparse rewards in multi-goal RL;
    \item We apply supervised learning on the relabeled data and introduce a joint loss for RL training. We also prove that minimizing the supervised loss using the model-based relabeled data is equivalent to optimizing a lower bound on the original multi-goal RL objective;
    \item Empirical results on several benchmark environments demonstrate that the proposed method, MHER, exceeds previous multi-goal RL algorithms in sample efficiency.
\end{itemize}

 \begin{figure}[t]
    \centering
    \subfigure[]{
    \begin{minipage}[t]{0.45\linewidth}
        \centering
        \includegraphics[width=1\linewidth]{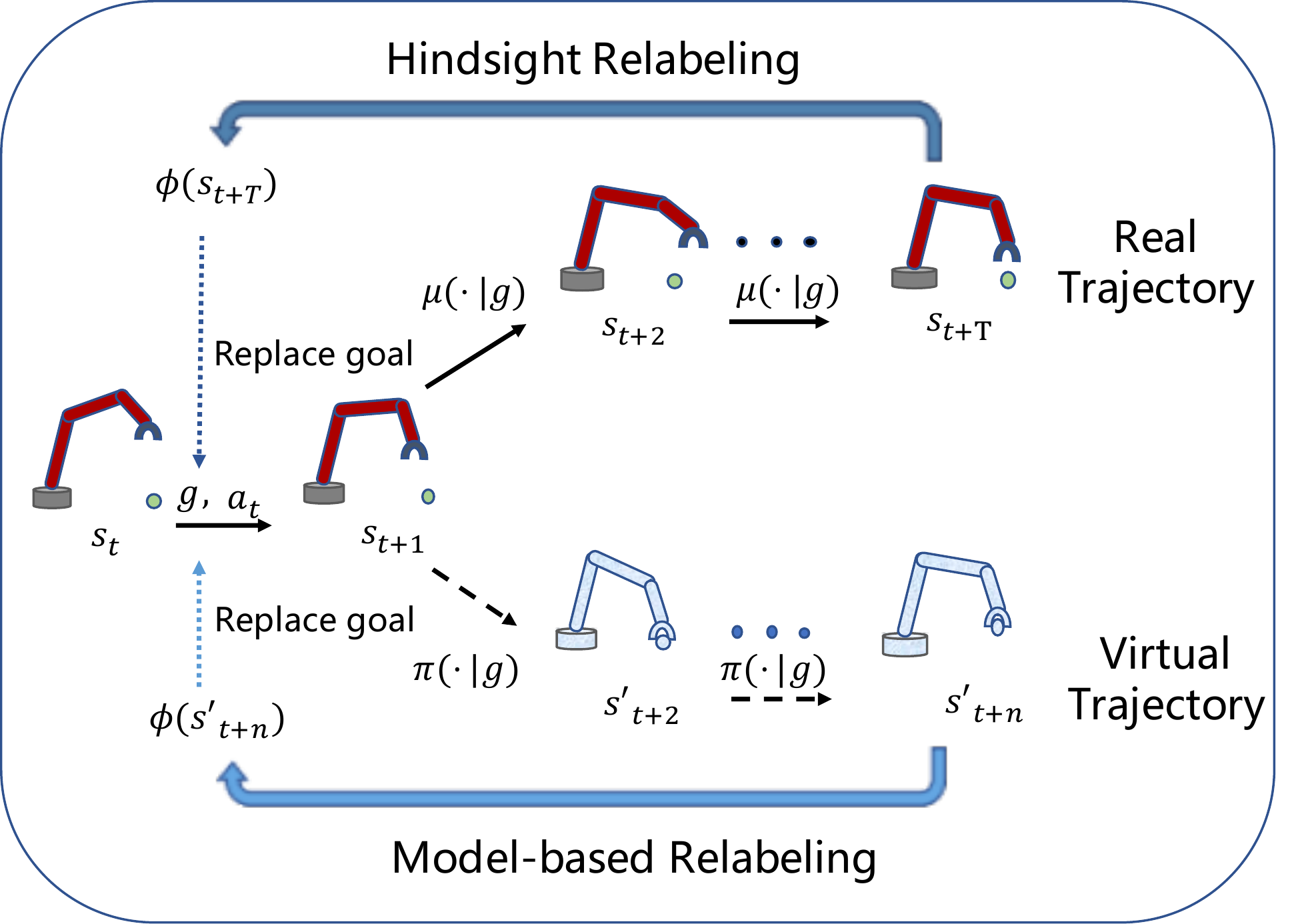}\\
    \end{minipage}%
    }%
    \subfigure[]{
    \begin{minipage}[t]{0.55\linewidth}
        \centering
        \includegraphics[width=1\linewidth]{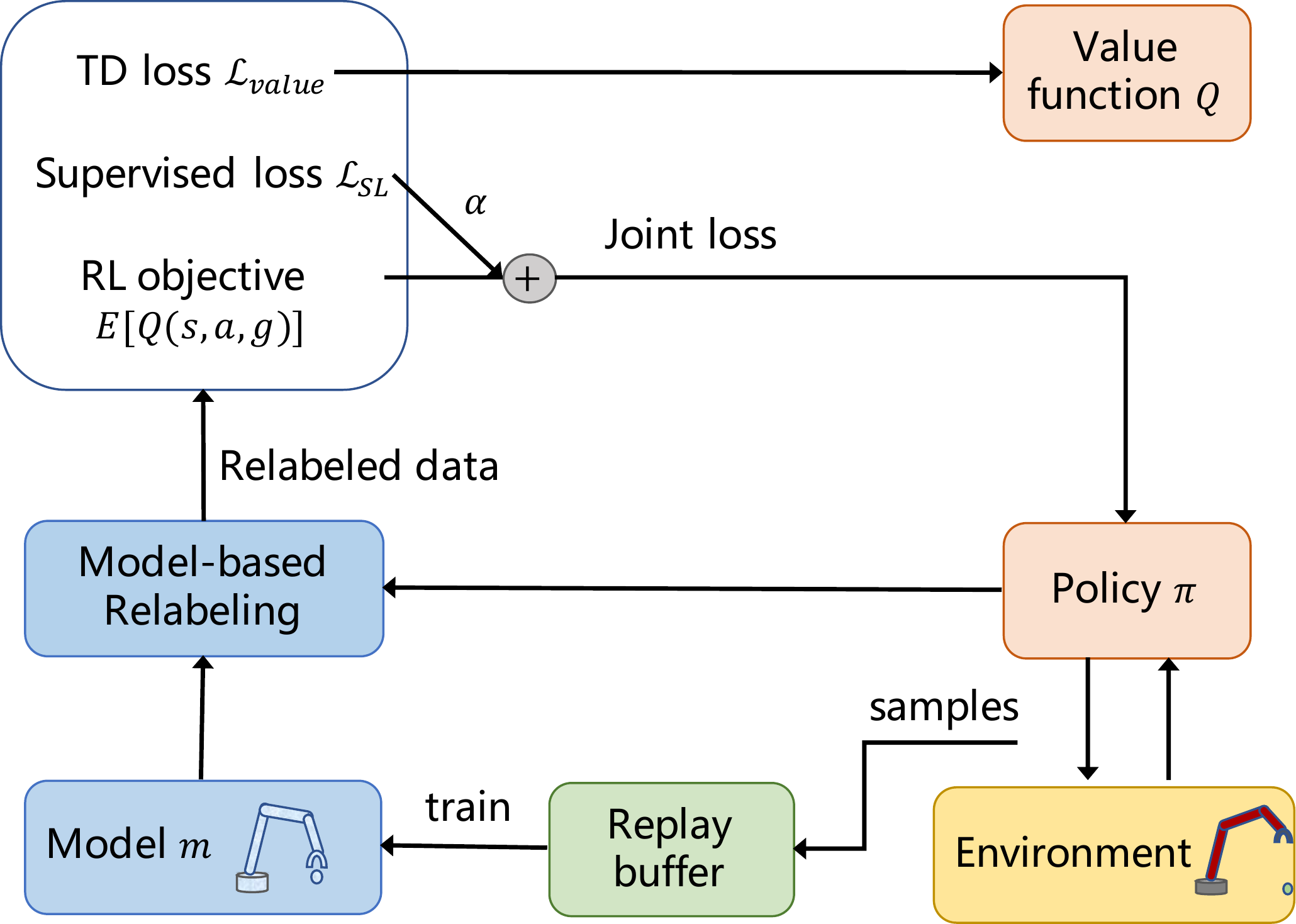}\\
    \end{minipage}%
    }%
    \caption{(a) Diagram of model-based relabeling. The real trajectory is collected by a past behavior policy $\mu$ and the virtual trajectory is generated from the interaction of the current policy $\pi$ and a trained dynamics model. Hindsight relabeling such as HER uses real achieved goals (e.g., $\phi(s_{t+T})$, $\phi$ is a state-to-goal mapping) to relabel, while model-based relabeling utilizes virtual achieved goals (e.g., $\phi(s'_{t+n})$). (b) Overall framework of MHER. Detailed descriptions can be found in Section \ref{sec:algorithm}.}
    \label{fig:diagram}
\end{figure}

\section{Related Work}
Our work concentrates on sample efficiency in multi-goal reinforcement learning (RL) with sparse rewards. Hindsight Experience Replay (HER) \cite{andrychowicz2017hindsight} introduces hindsight relabeling for multi-goal RL, opening up a new way to learn from failed experiences with sparse and binary rewards. Based on HER, a few studies have been investigated to find more efficient goal sampling ways. Curriculum-guided HER (CHER) \cite{fang2019curriculum} selects goals in a heuristic way to balance the diversity of selected goals and the proximity to original goals. Focusing on the long horizon problem in multi-goal RL, Maximum Entropy Goal Achievement \cite{pitis2020maximum} samples from the frontier of achieved goals and gradually increases the entropy of achieved goals. Recent woks \cite{DBLP:conf/nips/LiPA20,DBLP:conf/nips/EysenbachGLS20} view hindsight relabeling as the inverse RL and sample relabeling goals according to their cumulative return or value function. Considering the visual tasks, \cite{DBLP:conf/nips/SahniBAK19} leverages GAN to generate virtual goals for relabeling, but it requires collecting a special dataset of near-goal states. In contrast to previous works, relabeling goals of MHER are generated from the interaction between current policy and a learned dynamics model.

In addition to improving the goal sampling methods, other works introduce prioritization replay and supervised learning to address sample inefficiency in multi-goal RL. Energy-Based Prioritization proposes to more frequently replay trajectories with higher energy. Maximum Entropy-based Prioritization (MEP) \cite{zhao2019maximum} prioritizes experiences based on trajectory entropy. By incorporating a few expert demonstrations, \emph{goalGAIL} \cite{ding2019goal} significantly speeds up the convergence of policy via adversarial imitation learning. Self-supervised methods provides another simple but effective way for multi-goal RL. Policy Continuation with Hindsight Inverse Dynamics (PCHID) \cite{DBLP:conf/nips/SunLLZL19} extends hindsight inverse dynamics to the multi-step situation following dynamic programming, thus enabling learning in a self-imitated scheme. Goal-Conditioned Supervised Learning (GCSL) \cite{ghosh2021learning} provides theoretical guarantees that supervised learning from hindsight relabeled experiences optimizes a lower bound on the goal-oriented RL objective. Unlike PCHID and GCSL which both use real achieved goals for supervision, we leverage virtual goals for supervision and provide theoretical guarantees that supervised learning with virtual goals can lead to policy improvement.

Model-based RL algorithms have been studied for a long history and generally obtain higher sample efficiency over model-free algorithms \cite{atkeson1997comparison}. Dyna \cite{sutton1991dyna} generates virtual samples with a trained dynamics model to augment training data. Learned dynamics models can also be incorporated into model-free algorithms to accelerate learning of policies and value functions \cite{nagabandi2018neural}. Model-based value expansion (MVE) \cite{feinberg2018model} improves model-free value estimation with predictive transitions. To tackle with the model bias in learned models, STEVE \cite{buckman2018sample} uses an ensemble of models for a robust prediction. Model-Based Policy Optimization (MBPO) \cite{DBLP:conf/nips/JannerFZL19} proves a monotonic improvement with limited use of a predictive model. However, most of the model-based methods are designed for tasks with dense rewards. PlanGAN \cite{DBLP:conf/nips/CharlesworthM20} is the first algorithm to use GANs and models for planning in sparse-reward multi-goal tasks. But PlanGAN is a planning method and requires a huge amount of computation to simulate trajectories for selecting a single action, and it may also suffer from the accumulated model-based errors in the long planning horizon. Different from these works, our work only utilizes the virtual achieved goals to avoid training with full virtual states and alleviate the impact of model bias.

\section{Preliminaries}
\label{sec:preliminary}
\subsection{Multi-goal Reinforcement Learning}
Multi-goal reinforcement learning (RL) can be characterized by the tuple $\langle S, A, G, r, p, \gamma\rangle$, where $S, A, G, p, \gamma$ refer to state space, action space, goal space, transition function, and discount factor, respectively, and $r: S \times A \times G \rightarrow \mathbb{R}$ is the goal-conditioned reward function \cite{schaul2015universal}. A commonly used sparse reward function in multi-goal setting is defined as:
\begin{equation}
\label{equ:rewardfunction}
    r(s_t, a_t,g)=\begin{cases}0, & \left|\left|\phi(s_{t})-g\right|\right|^2_2 < \text{threshold} \cr -1, &\text{otherwise}\end{cases} ,
\end{equation}
where $\phi: S \rightarrow G$ is a mapping function from states to achieved goals and is assumed to be available \cite{andrychowicz2017hindsight}. Agents are required to learn a policy $\pi:S\times G \rightarrow A$ to maximize returns of reaching goals from the desired distribution $p(g)$:
\begin{equation}
\label{eq:objective}
    J(\pi)=E_{g\sim p(g),a_t \sim \pi, s_{t+1}\sim p(\cdot|s_t,a_t)}\big[\sum^{\infty}_{t=0} \gamma^t r(s_t, a_t, g) \big]
\end{equation}

\subsection{Hindsight Experience Replay}
\label{sec:HER}
Hindsight Experience Replay (HER) learns from failed experiences and tackles with sparse rewards in multi-goal RL. Given a trajectory $\tau = \{(s_t,a_t,g, r_t,s_{t+1}) \}_{t=1}^{T}$ of length $T$, HER alternates $g$ and $r_t$ in the $t$-th transition $(s_t,a_t,g, r_t,s_{t+1})$ with a future achieved goal in the same trajectory $g'=\phi(s_{t+k}), 1 \leq k \leq T-t$ and $r'_t=r(s_t,a_t,g')$ computed by Eq. \eqref{equ:rewardfunction}. After relabeling, transitions in failed trajectories can be assigned nonnegative rewards, therefore HER addresses the core issue of sparse rewards. HER can be combined with any off-policy algorithms such as DQN \cite{mnih2015human}, DDPG \cite{lillicrap2015continuous}, TD3 \cite{fujimoto2018addressing}, and SAC \cite{haarnoja2018soft}.

In our paper, we adopt the DDPG+HER framework following \cite{andrychowicz2017hindsight,fang2019curriculum}. DDPG is an off-policy actor-critic algorithm consisting of a deterministic policy $\pi$ and a value function $Q: S\times A\times G \rightarrow \mathbb{R}$. When collecting data with the policy $\pi$, Gaussian noise with zero mean and constant standard deviation is applied for exploration, as implemented in \cite{fujimoto2018addressing}. We denote the data distribution after hindsight relabeling as $B_h$. The value function $Q$ is updated to minimize the TD error:
\begin{equation}
\label{eq:loss_critic}
\mathcal{L}_{critic}=E_{(s_t,a_t,g', r'_t,s_{t+1})\sim B_h} \big[(y_t-Q(s_t,a_t, g'))^2 \big],
\end{equation}
where
\begin{equation*}
y_t = r'_t + \gamma Q (s_{t+1}, \pi (s_{t+1}, g'), g').
\end{equation*}
The policy $\pi$ is trained with policy gradient on the following loss:
\begin{equation*}
\mathcal{L}_{actor} = -E_{(s_t,g') \sim B_h} \big[ Q(s_t, \pi(s_t,g'), g') \big].
\end{equation*}

\section{Methodology}
In this section we will describe how our method, MHER, integrates environment dynamics into Hindsight Experience Replay \cite{andrychowicz2017hindsight} for training goal-conditioned policies. First we introduce the way we train the dynamics model that maps current state and action to the next state \cite{sutton1991dyna,feinberg2018model}. Then we propose a novel model-based relabeling technique based on the learned model and introduce a joint loss combining policy gradient \cite{silver2014deterministic,lillicrap2015continuous} and goal-conditioned supervised learning \cite{ghosh2021learning} on the model-based relabeled data. Finally, we describe the overall framework of MHER.

\subsection{Dynamics Models}
When considering deterministic dynamics, the most direct way of training dynamics models is to learn to predict the next state given current state and action. Let $m(s_t,a_t)$ denote a learned dynamics function that takes the current state $s_t$ and action $a_t$ as input and outputs an estimation of the next state $s_{t + 1}$. 
However, this model can be difficult to learn when the states $s_t$ and $s_{t+1}$ are very similar and the action has seemingly little
effect on the changes in continuous control environments~\cite{nagabandi2018neural}.
Therefore we follow~\cite{nagabandi2018neural} to learn a dynamics function that predicts the change between states by minimizing the following loss:
\begin{equation}
\label{eq:dynamicloss}
    \mathcal{L}_{model} = E_{(s_t,a_t,s_{t+1})\sim B} \big[\|(s_{t+1} -s_t) - m(s_t,a_t)\|_2^2\big],
\end{equation}
where the data $(s_t,a_t,s_{t+1})$ is sampled from the replay buffer $B$. Note that the dynamics model is independent of goals and rewards, therefore the data to train the model $m$ can also be sampled from the relabeled data distribution such as $B_h$.
With the trained dynamics model, current state $s_t$ and action $a_t$, we can predict the next state as $s_{t+1}=s_t + m(s_t,a_t)$. 



\subsection{Model-based Relabeling}
\label{sec:MBR}
The insight of model-based relabeling (MBR) is that states in the virtual trajectory generated by the dynamics model can also be viewed as achieved goals for the starting state, so it can be used for relabeling the original transition. As shown in Figure \ref{fig:diagram}, given a transition $(s_t, a_t, s_{t+1}, r_t, g)$ collected by a behavior policy, MBR starts at $s_{t+1}$ and interacts with the dynamics model $m$ using current policy $\pi$ for $n$ steps. After interaction with the model, we have a virtual trajectory of $\{s'_{t+i}, a'_{t+i}, s'_{t+i+1}\}_{i=0}^{n}$, where $i=0$ is the original transition and $a'_{t+i}=\pi(s'_{t+i},g), s'_{t+i+1}=s'_{t+i} + m(s'_{t+i}, a'_{t+i}), 1 \leq i\leq n$. Then, any state in the virtual trajectory implies an achieved goal guided by current policy. MBR randomly samples from the virtual achieved goals $g'=\phi(s'_{t+j}), 1 \leq j \leq n$ to relabel the original transition: $(s_t, a_t, s_{t+1}, r'_t, g')$, where $\phi$ is the state-to-goal mapping and $r'_t=r(s_t,a_t,g')$ is computed according to Eq. \eqref{equ:rewardfunction}. In our relabeled transitions, only the goal $g'$ is virtually generated and other items are real experiences, avoiding the usage of full virtual states for training the value function. More detailed descriptions are provided in Appendix \ref{ap:virtual_ags}. Note that the model-based interaction is driven by the current policy $\pi$ under original goals $g$, thus MBR gradually pushes the relabeling goals towards the desired goal as the policy $\pi$ improves. Theoretically, optimizing $J(\pi)$ is equivalent to minimizing an upper bound on the expected distance between virtual achieved goals and original desired goals. The proof can be found in Appendix \ref{ap:theorem_goals}. We will also verify this property empirically in Section \ref{sec:curriculum_relabeling}. 

Two major advantages of model-based relabeling over previous relabeling strategies are as follows: 
\begin{itemize}
    \item[(1)] MBR takes advantages of the current policy and environmental dynamics to generate more diverse goals for accelerating policy learning.
    \item[(2)] As the policy improves, relabeling goals will gradually approach the assigned targets, therefore an implicit curriculum of automatic relabeling is introduced.
\end{itemize}

\begin{algorithm}[tb]
	\caption{MHER Framework 
	\label{framwork}} 
	Given model-based interaction steps $n$, and $\alpha$ in the joint policy loss $\mathcal{L}_{joint}$\;
	Initialize policy $\pi$ and value function $Q$\;
	Warmup the dynamics model $m$ by minimizing $\mathcal{L}_{model}$ in Eq. (\ref{eq:dynamicloss}) with random samples\;
		\For {episode = $1,2,\ldots,M$}{
		    Sample a desired goal $g$\;
		    Collect a trajectory with the policy $\pi$ and save to the replay buffer $B$\;
		    Sample a minibatch $b$ from the replay buffer $:\{(s_t,a_t,s_{t+1},r_t,g)_i\}^N_{i=1}\sim B$\;
		    Update the dynamics model $m$ with $b$\;
		    \For {$i=1,2, \ldots, N$}{
		        $(s_t,a_t,s_{t+1},r_t,g) \leftarrow b_i$\;
		        \tcp{model-based relabeling}
		            $s'_{t+1}=s_{t+1}$, $v = \{s'_{t+1}\}$ \;
		            \For{$j=1,2,\ldots,n$}{
		                $a'_{t+j} = \pi(s'_{t+j},g)$\;
		                $s'_{t+j+1} = s'_{t+j} + m(s'_{t+j}, a'_{t+j})$\;
		                Append $s'_{t+j+1}$ to $v$ \;
		            }
		            Sample random future state $s\sim v$\;
		            Get virtual achieved goal $ g'=\phi(s)$\;
		            Recompute $r'=r(s_t,a_t,g')$ using reward function in Eq. \eqref{equ:rewardfunction}\;
		            $b_i \leftarrow (s_t, a_t, r', s_{t+1}, g')$\;
		    }
		    Update value function $Q$ with $b$ to minimize $\mathcal{L}_{value}$ in Eq. \eqref{eq:loss_value}\;
		    Update policy $\pi$ with $b$ to minimize $\mathcal{L}_{joint}$ in Eq. \eqref{eq:loss_joint} \;
		}
\end{algorithm}

\subsection{Learning Policy with Model-based Relabeled Data}
After model-based relabeling (MBR), we further perform policy update based on the relabeled data. Similar to \cite{ghosh2021learning}, the policy can be optimized through supervised learning on the model-based relabeled data. However, the difference is that we use virtually generated goals $g'$ for supervision and minimize a mean square error rather than optimize the maximum likelihood, which are substantially the same under certain assumptions. We denote the model-based relabeled transition distribution as $B_m$, then the supervised loss $\mathcal{L}_{SL}$ is defined as:
\begin{equation}
\label{eq:loss_MGSL}
    \mathcal{L}_{SL} = E_{(s_t,a_t,g') \sim B_m} \big[\|a_t - \pi(s_t,g')\|_2^2 \big],
\end{equation}
where $g'$ is relabeled using MBR. Theoretically, we prove that minimizing $\mathcal{L}_{SL}$ in Eq. (\ref{eq:loss_MGSL}) is equivalent to maximizing a lower bound of $J(\pi)$, which is formally presented below.
\begin{theorem}
\label{tm:lowerbound}
    Given a Diagonal Gaussian policy with a mean vector $\pi(s,g')$ and a non-zero positive constant variance $\sigma^2$, the discount factor $\gamma$, the real environmental dynamics $p(\cdot|s,a)$, a learned dynamics model $p_m(\cdot|s,a)$, the model-based relabeled data distribution $B_m$, and $n$-step model-based interactions, minimizing the supervised loss $\mathcal{L}_{SL}$ in Eq. (\ref{eq:loss_MGSL}) is equivalent to maximizing a lower bound on the multi-goal RL objective.
    \begin{equation*}
        J(\pi) \geq -\frac{n\gamma^n}{2\sigma^2} \mathcal{L}_{SL} + C_1 \epsilon_{m} + C_2 ,
    \end{equation*}
    where $C_1, C_2$ are two constants independent of policy $\pi$. $\epsilon_m$ is the model error bounded at each timestep $t$: $\epsilon_m=max_t E_{s\sim \pi_{D,t}} [D_{TV}(p(s'|s,a)\| p_m(s'|s,a))]$, and $\pi_D$ is the data collecting policy of $B_m$.
\end{theorem}
 The detailed proof is provided in Appendix \ref{ap:proof}.
 
%

Previous works update policy via either policy gradient \cite{andrychowicz2017hindsight,fang2019curriculum} or supervised learning \cite{ghosh2021learning} with the hindsight relabeled data introduced in Sec \ref{sec:HER}. In contrast to these works, we propose a joint loss combining policy gradient and the supervised loss $\mathcal{L}_{SL}$ for more efficient policy learning. Specifically, at each iteration during training, the policy is trained to minimize the following joint loss on the model-based relabeled data $B_m$:
\begin{equation}\label{eq:loss_joint}
\begin{aligned}
    \mathcal{L}_{joint} = & - \mathbb{E}_{(s_t, g') \sim B_m} \big[ Q(s_t,\pi(s_t, g'),g') \big] \\
    &+ \alpha \mathbb{E}_{(s_t,a_t,g') \sim B_m}\big[\|a_t - \pi(s_t,g')\|_2^2 \big],
\end{aligned}
\end{equation}
where $\alpha > 0$ is the weight balancing the expected return and the supervised loss $\mathcal{L}_{SL}$. The Q-function is updated to minimize the TD error on the model-based relabeled data $B_m$:
\begin{equation}
\label{eq:loss_value}
\mathcal{L}_{value}=E_{(s_t,a_t,g', r'_t,s_{t+1})\sim B_m} \big[(y_t-Q(s_t,a_t, g'))^2 \big],
\end{equation}
where $y_t = r'_t + \gamma Q (s_{t+1}, \pi (s_{t+1}, g'), g')$.

\subsection{Algorithm}
\label{sec:algorithm}
The overall framework of model-based hindsight experience replay (MHER) is presented in Algorithm \ref{framwork} and Figure \ref{fig:diagram} (b). First, the dynamics model $m$ is trained to minimize the loss $\mathcal{L}_{model}$ using Eq. \eqref{eq:dynamicloss} in both warm-up and training periods. For every episode, we sample a goal $g$ from the desired goal distribution and collect a trajectory to the replay buffer $B$ using current policy. After collecting data, we sample a minibatch $b$ from the replay buffer $B$. For each transition in the minibatch, we leverage the current policy $\pi$ to rollout a $n$-step trajectory with the learned dynamics model $m$ and perform model-based relabeling as explained in Section \ref{sec:MBR}. 
After MBR, the minibatch $b$ belongs to the model-based relabeled distribution $B_m$ and is sent for training the Q network and the policy network. The Q network is updated according to Eq. \eqref{eq:loss_value} and the policy is trained to minimize the joint loss $\mathcal{L}_{joint}$ in Eq. (\ref{eq:loss_joint}).


\section{Experiments}
We conduct experiments on both continuous Point2D and Mujoco environments and
compare the performance of MHER against a number of leading multi-goal RL algorithms for sparse reward environments. We also demonstrate the effectiveness of our goal-relabeling method by visualizing the distribution of relabeling goals. 


\subsection{Experimental Settings}
\label{sec:env_and_setting}

\paragraph{Environments}Our experimental environments consist of two Point2D environments, one Sawyer robots, and two Mujoco robots modified from OpenAI Gym \cite{brockman2016openai}. All of the five environments' states, actions, and goals are continuous. In the first two point-based environments the blue point aims to reach the green circle. The other four environments (FetchReach-v1, SawyerReachXYZEnv-v1, Reacher-v2) control a robot to reach a target position in the goal space. More detailed task description can be found in Appendix \ref{ap:task_description}.

\paragraph{Baseline Implementation}All the implementations of baseline algorithms are taken from their open-source code except GCSL, which we implement a deterministic version to fairly compare with other algorithms. For GCSL \cite{ghosh2021learning}, we only maintain a policy network and train it to minimize the loss $\mathcal{L}_{GCSL} = E_{(s_t,a_t,g') \sim B_h} \big[\|a_t - \pi(s_t,g')\|_2^2 \big]$, where $g'$ is relabeled using future achieved goals similar to HER and $B_h$ refers to the data distribution after hindsight relabeling.

\paragraph{Implementation of MHER}As for the implementation of MHER, actor and critic networks are both 3-layer fully connected networks with 256 units each layer. All the networks are updated with Adam optimizer, learning rate $1\times 10^{-3}$, and Polyak averaging coefficient of $0.9$. To encourage exploration, the probability of random action is set as $0.3$, and the scale of Gaussian noisy is $0.2$. Following \cite{andrychowicz2017hindsight,fang2019curriculum}, we perform model-based relabeling with a probability of $0.8$. Note that the portion of data without MBR is not sent for goal-conditioned supervised learning, which can be implemented with a mask.

\paragraph{Training Settings}We train all the algorithms for 30 epochs with one rollout worker. Each epoch contains $1$ (Point2DLargeEnv-v1, Point2D-FourRoom-v1), $5$ (FetchReach-v1, SawyerReachXYZEnv-v1), or $15$ (Reacher-v2) episodes according to the difficulty of the environments. Every episode contains $100$ interaction steps with the environment. At the end of each episode, we train all the algorithms for $5$ batches with a batch size of $64$. 

\paragraph{Evaluation Settings}After each epoch, we evaluate every algorithm for 100 episodes, and report the median test success rate and the interquartile range across $5$ random seeds. To train the dynamics model $m$ for MHER, we use a fully connected network with $4$ hidden layers and $256$ neurons each layer. In the warmup period, we train $m$ for $100$ updates with a batch size of $512$ and a learning rate of $0.001$. When training with MHER, we update $m$ with the sampled minibatch for $2$ times. More detailed hyperparameters are provided in Appendix \ref{ap:implementation_detail}.

\begin{figure*}[tb]
    \centering
    \includegraphics[width=1\linewidth]{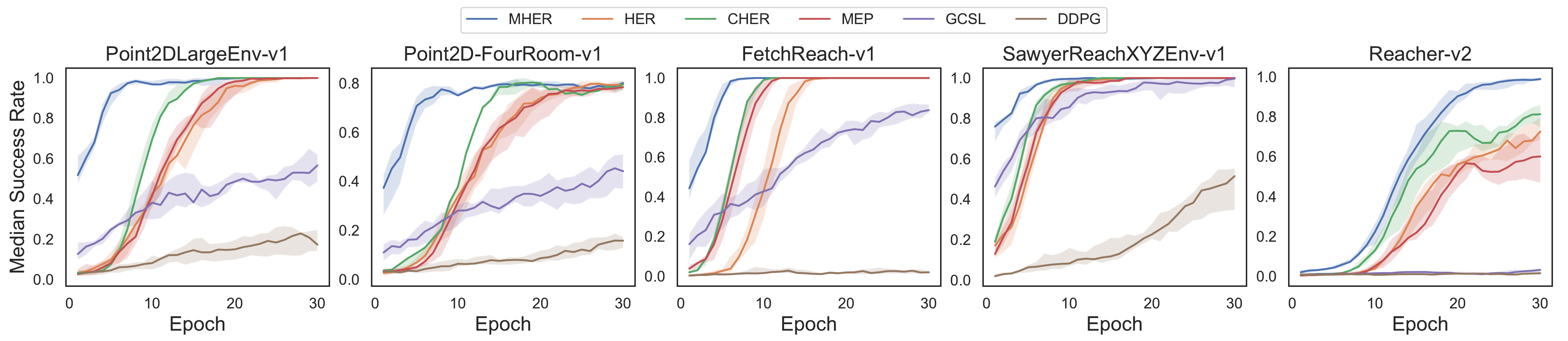}
    \caption{Median test success rate (line) with interquartile range (shaded area) on Point2D and Mujoco environments acorss 10 random seeds.}
    \label{fig:results}
\end{figure*}

\subsection{Benchmark Results}
We compare MHER with baselines such as DDPG, HER \cite{andrychowicz2017hindsight}, CHER \cite{fang2019curriculum}, MEP \cite{zhao2019maximum}, and GCSL \cite{ghosh2021learning} in five benchmark environments. For MHER, the model-based interaction step is set as $5$ and the parameter $\alpha$ in the joint loss is set as $3$. The median test success rates are reported in Figure \ref{fig:results}. It is apparent that MHER achieves significantly higher performance than those baselines with a much faster learning speed. The results also show that DDPG learns slowly in all the environments, while HER, CHER, MEP are three effective baselines. Besides, GCSL's performance in different environments is not consistent, e.g., in Reacher-v2 it is very close to DDPG but in SawyerReachXYZEnv-v1 it outperforms DDPG by a large margin. 

In Figure \ref{fig:merge} (a) and \ref{fig:merge} (b), we study how the parameters $\alpha$ and model-based interaction steps $n$ impact MHER's performance. The parameter $\alpha$ impacts the weight of the supervised policy loss in $\mathcal{L}_{joint}$ when optimizing the policy, and $\alpha = 0$ indicates learning without the supervised loss. The results in Figure \ref{fig:merge} (a) suggest that as $\alpha$ increases, performance increases when $\alpha \leq 3$ and decreases when $\alpha > 3$. As for the model-based interaction steps, more steps may contain more long-distance information but too many steps can lead to irrelevant goals. The results in Figure \ref{fig:merge} (b) verify our analysis and the horizon of $5$ achieves the best performance in the candidate set $\{0,1,3,5,7\}$. More experimental results are provided in Appendix \ref{ap:extra_experiment}.

\begin{figure}[t]
    \centering
    \includegraphics[width=1\linewidth]{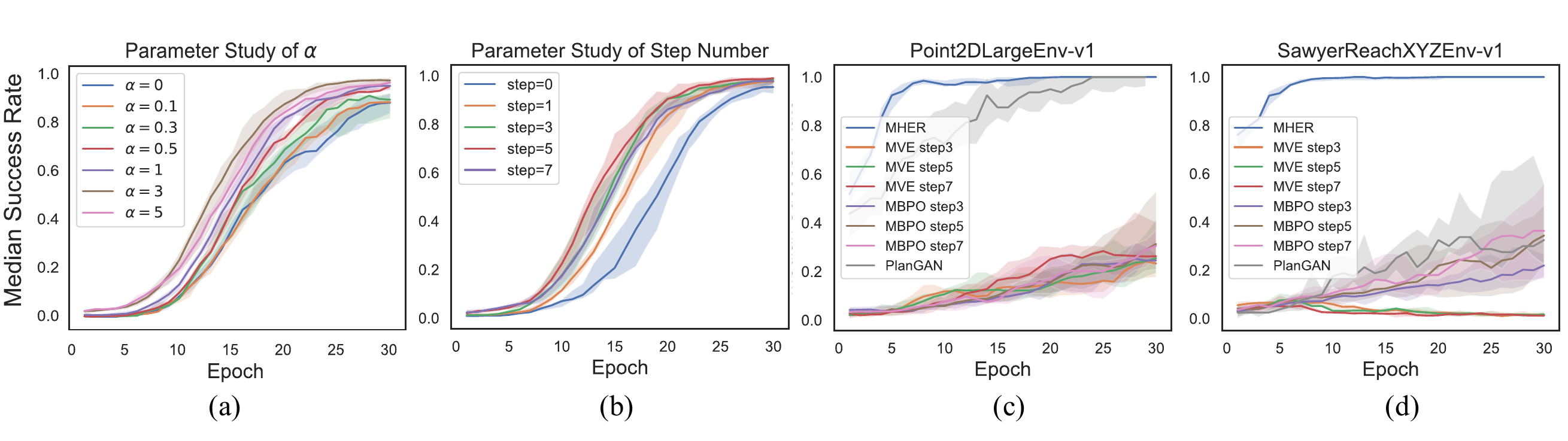}
    \caption{(a)(b) Parameter study of $\alpha$ and step number in Reacher-v2 environment. (c)(d) Comparison results with model-based based reinforcement learning baselines, MVE and MBPO, in Point2DLargeEnv-v1 and SawyerReachXYZEnv-v1 environments.}
    \label{fig:merge}
\end{figure}

\begin{figure}[t]
    \centering
    \includegraphics[width=1\linewidth, trim=10 0 0 0]{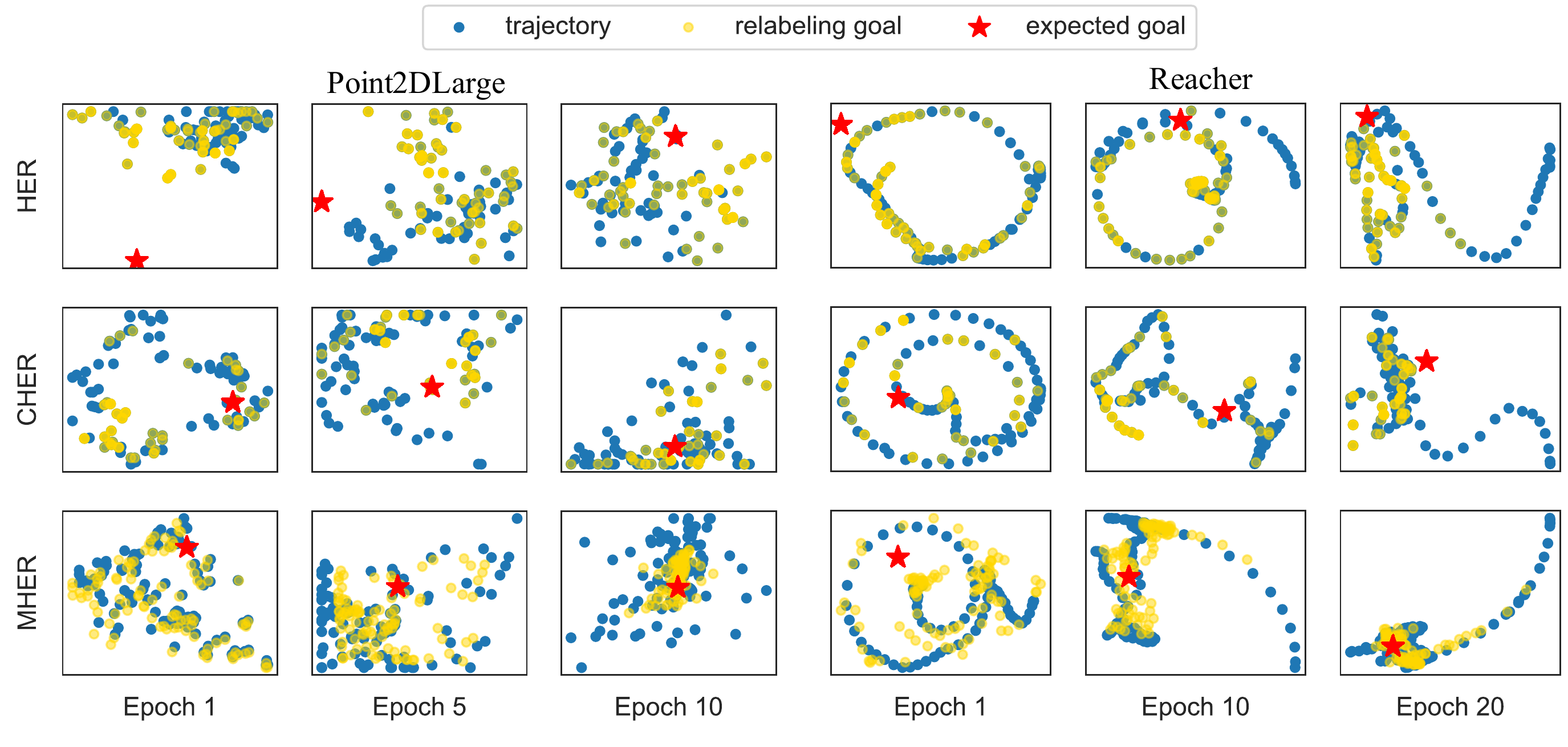}
    \caption{Relabeling goal distributions of HER, CHER, and MHER in Point2DLargeEnv-v1 and Reacher-v2. Blue points are real states in a trajectory, and red stars are their expected goals. Yellow points are relabeling goals for each transition in the trajectory.}
    \label{fig:goals}
\end{figure}

\subsection{Results of Model-based RL baselines}
We also make fair comparison with model-based RL baselines in Figure \ref{fig:merge} (c) and (d), including Model-based Value Expansion (MVE) \cite{feinberg2018model} and Model-based Policy Optimization (MBPO) \cite{DBLP:conf/nips/JannerFZL19}. Implementation details of MVE and MBPO are provided in Appendix \ref{ap:implementation_detail}. From the experimental results, we can conclude that previous model-based RL methods contribute little to multi-goal RL tasks with sparse rewards. In contrast to these algorithms, MHER can exploit environmental dynamics to learn policies more efficiently in the sparse reward setting. We also compare with PlanGAN \cite{DBLP:conf/nips/CharlesworthM20}, a planning method rather than a RL method, in Appendix \ref{ap:extra_experiment}. 
The results show that MHER can achieve comparable or higher performance compared with PlanGAN, although PlanGAN requires extremely large amount of computations (detailed analysis is provided in Appendix \ref{ap:comparison_with_model_based}) to simulate trajectories for selecting actions.

\subsection{Relationship with Curriculum Goal Relabeling}
\label{sec:curriculum_relabeling}
Prior work, Curriculum-guided HER (CHER) \cite{fang2019curriculum} introduces an explicit curriculum for automatically selecting relabeling goals. Nevertheless, the curriculum is hand-designed and needs to be adjusted through hyperparameters. 

In Figure \ref{fig:goals}, we compare the relabeling goals of HER, CHER, and MHER in Point2DLargeEnv-v1 and Reacher-v2 environments. We can observe that HER selects goals randomly along the future achieved goals, leading to more goals in the end of trajectories. Besides, CHER selects goals to balance the diversity of selected goals and the proximity to the expected goal. Moreover, HER and CHER only select real achieved goals, thus the set of sampling goals is limited. On the contrary, MHER generates virtual relabeling goals according to the dynamics model and the current policy under the original desired goal. In the early stage, relabeling goals of MHER are near achieved goals. As the policy improves, they gradually approach their expected goals. The learning process of MHER is automatically adjusted by the policy, therefore MHER can substantially achieve more efficient curriculum learning.


\begin{figure}[t]
    \centering
    \includegraphics[width=0.95\linewidth]{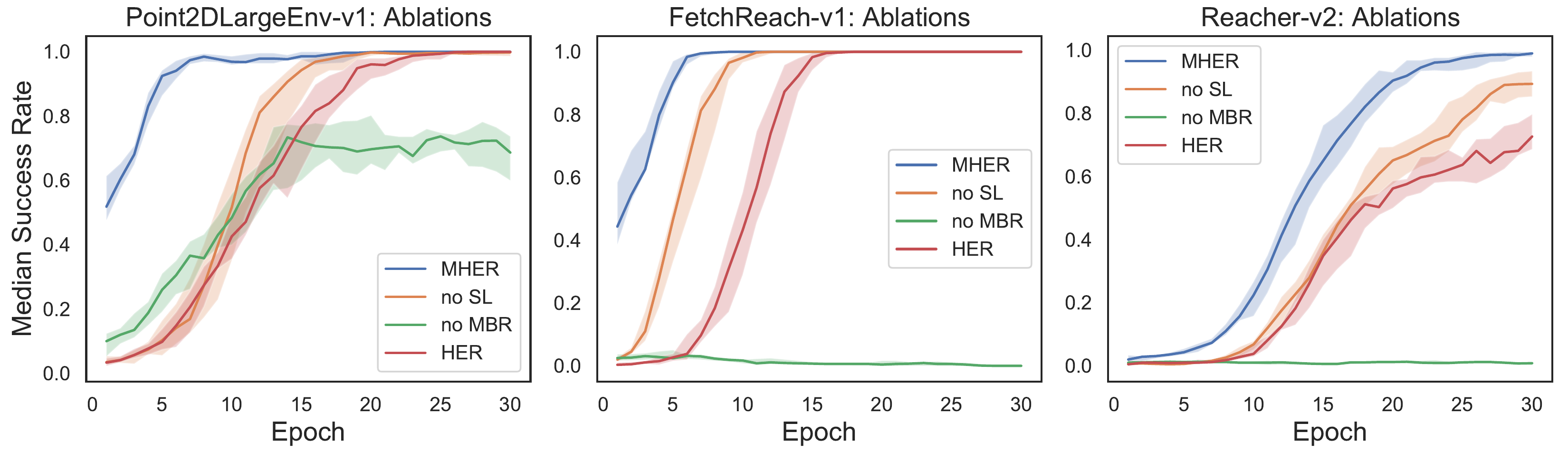}
    \caption{Ablation studies in Point2DLargeEnv-v1, FetchReach-v1, and Reacher-v2.}
    \label{fig:ablations}
\end{figure}

\subsection{Ablation Studies}
To analyze the importance of \emph{model-based relabeling} (MBR) and \emph{supervised learning on the MBR data} (denoted as SL) in the MHER framework, we design ablation experiments to compare MHER variants with HER. For MBR and SL, the number of model-based interaction steps is $5$ and $\alpha=3$ by default. We experiment with the following settings:
\begin{itemize}
    \item \emph{MHER}: DDPG+MBR+SL;
    \item \emph{no SL}: DDPG+MBR, which is equivalent to $\alpha=0$;
    \item \emph{no MBR}: DDPG+SL, adding an auxiliary task to DDPG that only uses model-based relabeled transitions to minimize the supervised policy loss. 
\end{itemize}

Empirical results in Figure \ref{fig:ablations} demonstrate that MBR is more important than SL in the MHER framework. DDPG+MBR learns faster than HER, but DDPG+SL learns very slowly in FetchReach-v1 and Reacher-v2, and cannot converge to $100\%$ success rate in Point2DLargeEnv-v1, which indicates that SL is less effective compared to MBR. By combining MBR and SL, the performance is significantly improved. The intuition behind the results is that MBR and SL achieve mutual improvement under the MHER framework, because MBR provides goals following an efficient curriculum to train the policy, and the supervised policy loss further improves the policy that guides the curriculum of MBR.

\section{Conclusion}
In this paper, we propose the framework of \emph{Model-based Hindsight Experience Replay} (MHER) to improve sample efficiency in multi-goal RL with sparse rewards. In MHER, we introduce virtual goals for goal relabeling and policy improvement, corresponding to model-based relabeling (MBR) and a joint loss combining policy gradient and supervised learning with MBR data. Experiments in a range of continuous multi-goal tasks demonstrate that MHER achieves significantly higher sample efficiency than previous goal-conditioned works such as HER, CHER, MEP and GCSL. Moreover, we show that MHER is efficient in the following aspects: (1) virtual goals generated from interaction with the trained model are not limited to real experiences; (2) generated goals follow an efficient curriculum guided by the policy; (3) policy learning takes advantage of both reinforcement learning and supervised improvement. In the future, we would like to work on more efficient works inspired by MHER.

\newpage

\appendix


\section{Theoretical Analysis}
\subsection{Proof of Theorem \ref{tm:lowerbound}}
\label{ap:proof}
\begin{proof}
Following GCSL \cite{ghosh2021learning}, we denote a real trajectory as $\tau = \{s_0,a_0,\ldots, s_T, a_T\}$. The expected goal is $g$, and the real dynamics is $p(s_{t+1}|s_t,a_t)$. We assume the initial state $s_0$ is uniformly distributed in the state space. To simplify the analysis, the relabeling goal is defined as the achieved goal of the last state $\mathcal{G}(\tau) \triangleq \phi(s_T)$, where $\phi$ is the state-to-goal mapping and $\mathcal{G}$ is the mapping from trajectories to relabeling goals. This simplification is also adopted in \cite{ghosh2021learning}. Then we denote a virtual trajectory generated by the trained dynamics model $p_m(s'_{t+1}|s'_t,a'_t)$ as $\tau_m = \{s'_0, a'_0, \ldots, s'_n, a'_n\}, s'_0=s_0$, and suppose $T\geq n$. The model-based relabeling goal is defined as $\mathcal{G}(\tau_m)\triangleq \phi(s'_n)$. We utilize a sparse reward function: $r(s_t,a_t,g) = 1[\phi(s_t) = g]$, therefore the reward at each timestep is bounded $r_t \in[0,1]$. Other variables are defined in Section \ref{sec:preliminary}.

In our setting, the original multi-goal RL objective is:
\begin{equation*}
    J(\pi)=E_{g\sim p(g),a_t \sim \pi, s_{t+1}\sim p(\cdot|s_t,a_t)}\big[\sum^{T}_{t=0} \gamma^t r(s_t, a_t, g) \big] .
\end{equation*}
Similarly, the expected return of the virtual trajectory $\tau_m$ is:
\begin{equation*}
    J_m(\pi)=E_{g\sim p(g),a_t\sim\pi, s_{t+1}\sim p_m(\cdot|s_t,a_t)}[\sum_{t=0}^{n} \gamma^tr(s_t,a_t,g)] .
\end{equation*}
Following MBPO \cite{DBLP:conf/nips/JannerFZL19}, we first bound model error using $\epsilon_m$.
As our model-based approach collects data with current policy $\pi$, the policy shift $\epsilon_{\pi}=0$. Next, we apply Lemma \ref{lemma:J} to bound $J(\pi)$ with $J_m(\pi)$ and absorb constants into $C_1$ as:
\begin{equation*}
    J(\pi) \geq J_m(\pi) + C_1 \epsilon_m ,
\end{equation*}
where $C_1=-\frac{2\gamma r_{max}}{(1-\gamma)^2}=-\frac{2\gamma}{(1-\gamma)^2}$ is derived from Lemma \ref{lemma:J}.

In GCSL, agents optimize the goal-reaching objective, i.e., maximizing the probability of reaching $g$ at the last step:
\begin{equation*}
    \hat J(\pi) = E_{g\sim p(g), a_t \sim \pi, s_{t+1} \sim p(\cdot|s_t,a_t)}[1|\mathcal{G}(\tau)=g].
\end{equation*}
With Lemma \ref{lemma1}, we can also obtain the goal-reaching return bound for model-based return $J_m(\pi)$:
\begin{equation*}
\begin{aligned}
        J_m(\pi) \geq \gamma^n E_{g\sim p(g),a_t \sim \pi, s_{t+1}\sim p_m(\cdot|s_t,a_t)}\big[1[\mathcal{G}(\tau_m) = g] \big] ,
\end{aligned}
\end{equation*}
and denote the inequality as $J_m(\pi) \geq \gamma^n \hat J_m(\pi)$. 
For $\hat J_m(\pi)$, we follow Theorem 4.1 in \cite{ghosh2021learning} and have:
\begin{equation*}
\begin{aligned}
    \hat J_m(\pi) &\geq E_{a_t \sim \pi, s_{t+1} \sim p_m(\cdot|s_t,a_t)} \big[ \sum_{t=0}^n log\pi(a_t|s_t,\mathcal{G}(\tau_m))  \big] + C_2
\end{aligned}
\end{equation*}
As we collect data using the current policy $\pi$, we don't have the policy variation term. The right side of this inequality means using relabeled data for maximum likelihood estimation, therefore the policy $\pi$ is required to be stochastic and select actions with non-zero probability. 

Combining the intermediate results, we can conclude that:
\begin{equation*}
\begin{aligned}
    J(\pi) &\geq J_m(\pi) + C_1 \epsilon_m \geq \gamma^n \hat J_m(\pi) +  C_1 \epsilon_m \\
    & \geq \gamma^n E_{a_t \sim \pi, s_{t+1} \sim p_m(\cdot|s_t,a_t)} \big[ \sum_{t=0}^n log\pi(a_t|s_t,\mathcal{G}(\tau_m)) \big] + C_1 \epsilon_m + C_2
\end{aligned}
\end{equation*}
Then, we take the Diagonal Gaussian policy with mean vector $\pi(s_t,g)$ of dimension $|a|$ and non-zero positive constant variance $\sigma^2$:
\begin{equation*}
    P(a_t|s_t,g)=\frac{1}{(\sigma \sqrt{2Pi})^{|a|}} e^{\frac{-\|a_t-\pi(s_t,g)\|_2^2}{2\sigma^2}} .
\end{equation*}
Note that we use $Pi$ to represent 'pi' in the Gaussian distribution and to distinguish from the policy $\pi$. $Pi$ is a constant and we will absorb it in $C_2$ in the following derivation.
Denoting the model-based relabeled data distribution as $B_m$ and relabeled goal as $g'$, we have:
\begin{equation*}
    \begin{aligned}
    J(\pi) &\geq -\frac{\gamma^n}{2\sigma^2} E_{(a_t,s_t,g')^n_{t=0} \sim B_m} \big[ \sum_{t=0}^n \|a_t -\pi(s_t,g')\|_2^2 \big] + C_1 \epsilon_m + C_2 \\
    &= -\frac{n\gamma^n}{2\sigma^2} E_{(a_t,s_t,g') \sim B_m} \|a_t -\pi(s_t,g')\|_2^2  + C_1 \epsilon_m + C_2
    \end{aligned}
\end{equation*}
This completes the proof.
\end{proof}

\subsection{Useful Lemmas}
\begin{lemma}[Finite Horizon Return Bound]
\label{lemma:finite_return_bound}
Suppose the model error is bounded as $max_t E_{s\sim p_1^t(s)} D_{TV}(p_1(s'|s,a)\| p_2(s'|s,a))\leq \epsilon_m$, and $max_s D_{TV}(\pi_1(a|s)\| \pi_2(a|s))\leq \epsilon_\pi$. And We denote $J^{(T)}(\pi)$ as expected return of finite horizon $T$ under policy $\pi$, and $r_{max}>0$ is the maximum reward. Then the difference of finite horizon returns are bounded as:
\begin{equation*}
    |J^{(T)} (\pi_1) - J^{(T)} (\pi_2)| \leq 2r_{max} \sum_{t=0}^T \gamma^t [t(\epsilon_m + \epsilon_{\pi}) + \epsilon_{\pi}]
\end{equation*}
\end{lemma}
\begin{proof}
    This lemma can be derived by Lemma B.3 in \cite{DBLP:conf/nips/JannerFZL19}, we also provide a sketch proof.
    \begin{equation*}
    \begin{aligned}
        |J^{(T)} (\pi_1) - J^{(T)} (\pi_2)| &= |\sum_{s,a}\sum_{t} \gamma^t (p_1^t(s,a) - p_2^t(s,a))r(s,a)| \\
        &\leq r_{max} \sum_{t}\sum_{s,a} \gamma^t |p_1^t(s,a) - p^t_2(s,a)| \\
        &=2r_{max} \sum_{t}\gamma^t D_{TV}(p^t_1(s,a)\|p^t_2(s,a)) \\ 
        &\leq 2r_{max} \sum_{t}\gamma^t[D_{TV}(p^t_1(s)\|p^t_2(s)) + max_{s}D_{TV}(\pi_1(a|s)\|\pi_2(a|s))] \\
        &\leq 2r_{max} \sum_{t=0}^T \gamma^t [t(\epsilon_m + \epsilon_{\pi}) + \epsilon_{\pi}]
    \end{aligned}
    \end{equation*}
where $D_{TV}(p^t_1(s)\|p^t_2(s)) \leq t(\epsilon_m + \epsilon_{\pi})$ is also prove in Lemma B.3 in \cite{DBLP:conf/nips/JannerFZL19}. The branched returns bound (Lemma B.3 in \cite{DBLP:conf/nips/JannerFZL19}) is a special case of our Lemma \ref{lemma:finite_return_bound} when horizon $T \to \infty$.
\end{proof}

\begin{lemma}
\label{lemma:J}
    Let $J^{(\infty)}(\pi)$ denote $J(\pi)$ with infinite horizon, where the superscript refers to the horizon. $J(\pi)=J^{(T)}(\pi), J_m(\pi)=J_m^{(n)}(\pi), T\geq n$. Given model error $\epsilon_m$ and policy shift $\epsilon_{\pi}$ defined in \cite{DBLP:conf/nips/JannerFZL19}, non-negative reward function $r(s_t,a_t) \geq 0$, and the maximum reward $r_{max}$, $J(\pi)$ can be bounded as:
    \begin{equation*}
        J(\pi) \geq J_m(\pi) - \frac{2r_{max}}{(1-\gamma)^2}( \gamma \epsilon_m + 2\epsilon_{\pi} ),
    \end{equation*}
\end{lemma}
\begin{proof}
    Let $\pi_D$ denote the data collecting policy, we follow the proof of Theorem A.1 in \cite{DBLP:conf/nips/JannerFZL19} and have:
    \begin{equation*}
        J^{(T)}(\pi) - J^{(T)}_m(\pi) = \underbrace{J^{(T)}(\pi) - J^{(T)}(\pi_D)}_{L1} + \underbrace{J^{(T)}(\pi_D) - J^{(T)}_m(\pi)}_{L2}
    \end{equation*}
    For $L_1$ and $L_2$, we apply Lemma \ref{lemma:finite_return_bound} to derive the following bound (note that there should be no model error in $L_1$ term): 
    \begin{equation*}
    \begin{aligned}
        J^{(T)}(\pi) - J^{(T)}_m(\pi)
        &\geq   \underbrace{-2r_{max}\sum_{t=0}^{T} \gamma^t\big[ t\epsilon_{\pi} + \epsilon_{\pi} \big]}_{L_1} \underbrace{- 2r_{max}\sum_{t=0}^{T} \gamma^t\big[ t(\epsilon_m +\epsilon_{\pi}) + \epsilon_{\pi} \big]}_{L_2} \\
         &= - 2r_{max}\sum_{t=0}^{T} \gamma^t\big[ t(\epsilon_m +2\epsilon_{\pi}) + 2\epsilon_{\pi} \big],
    \end{aligned}
    \end{equation*}
    To verify the above inequality, let $T\rightarrow\infty$ and leverage the properties $\sum_{t=0}^{\infty}\gamma^t t \leq \frac{\gamma}{(1-\gamma)^2}$ and $\sum_{t=0}^{\infty}\gamma^t \leq \frac{1}{1-\gamma}$, then we can get MBPO performance bound (Theorem A.1 in \cite{DBLP:conf/nips/JannerFZL19}). Utilizing the assumptions $r_t\in[0,1], T\geq n$ and properties $\epsilon_m \geq 0, \epsilon_{\pi} \geq 0, J_m^{(T)}\geq J_m^{(n)}$, we can then prove the lower bound
    \begin{equation*}
    \begin{aligned}
    J^{(T)}(\pi) &\geq J^{(T)}_m(\pi) - 2r_{max}\sum_{t=0}^{T} \gamma^t\big[t(\epsilon_m +2\epsilon_{\pi}) + 2\epsilon_{\pi} \big] \\
    &\geq J^{(T)}_m(\pi) - 2r_{max}\sum_{t=0}^{\infty} \gamma^t\big[ t(\epsilon_m +2\epsilon_{\pi}) + 2\epsilon_{\pi} \big] \\
    &\geq J^{(n)}_m(\pi) - \frac{2 r_{max}[\gamma\epsilon_m + 2 \gamma \epsilon_{\pi}+2\epsilon_{\pi}(1-\gamma)]}{(1-\gamma)^2} \\
    &\geq J_m^{(n)}(\pi) - \frac{2r_{max}}{(1-\gamma)^2} ( \gamma \epsilon_m + 2\epsilon_{\pi}  ).
    \end{aligned}
    \end{equation*}
    Alternating $J_m^{(n)}(\pi), J^{(T)}(\pi)$ with $J_m(\pi), J(\pi)$ completes the proof.
\end{proof}

\begin{lemma}[Goal-reaching Return Bound]
\label{lemma1}
    Given trajectories $\tau = \{s_0,a_0,\ldots, s_T, a_T\}$ and reward function $r(s_t,a_t,g) = 1[\phi(s_t) = g]$, the goal-reaching objective $\gamma^T\hat J(\pi)$ is a lower bound of the original multi-goal objective $J(\pi)$.
\end{lemma}
\begin{proof}
    Leveraging the definition of $J(\pi)$, $\hat J(\pi)$, and the reward function $r(s_t,a_t,g) = 1[\phi(s_t) = g]$, we have:
    \begin{equation*}
    \begin{aligned}
    J(\pi)&=E_{g\sim p(g),a_t \sim \pi, s_{t+1}\sim p(\cdot|s_t,a_t)}\big[\sum^{T}_{t=0} \gamma^t r(s_t, a_t, g) \big] \\
    & \geq E_{g\sim p(g),a_t \sim \pi, s_{t+1}\sim p(\cdot|s_t,a_t)}\big[\gamma^T r(s_T, a_T, g) \big] \\
    & = \gamma^T E_{g\sim p(g),a_t \sim \pi, s_{t+1}\sim p(\cdot|s_t,a_t)}\big[1[\mathcal{G}(\tau) = g] \big] \\
    & = \gamma^T \hat J(\pi).
    \end{aligned}
    \end{equation*}
\end{proof}

\subsection{Theoretical Analysis of Virtual Achieved Goals}
\label{ap:theorem_goals}

We provide formal analysis on the expected distance between virtual achieved goals and original desired goals. First, we define the \emph{expected distance} ($ED_{\pi}(g)$) as the sum of the product of distance and the discounted visitation frequencies:
\begin{equation*}
    ED_{\pi}(g) = \sum_s \rho_\pi(s) \|\phi(s)-g\|_2^2
\end{equation*}
where g refer to original desired goals, $\phi$ is a state-to-goal mapping, and $\rho_\pi(s)=P(s_0=s)+\gamma P(s_1 = s|\pi)+\gamma^2 P(s_2 = s|\pi)+..., s_0\sim p(s_0)$ and actions are sampled according to $\pi$. The expected distance $ED_{\pi}(g)$ measure the proximity of trajectories generated by the policy $\pi$ and the desired goal $g$. The following theory demonstrates the expected distance in the model-based relabeling can be minimized by optimizing the policy.

\begin{theorem}[Expected Distance Bound, Infinite Horizon]
\label{tm:expected_distance}
    Suppose the state space is finite, the policy $\pi$ is optimized to maximize $J(\pi)$, the reward function is sparse $r(s,a,g)=1[\|\phi(s)-g\| \leq \epsilon]$, $\epsilon$ is a small positive threshold, the goal space is bounded, i.e., $\|g-g'\|_2^2 \leq C_g, \forall g,g' \in \mathcal{G}, C_g \geq 0$. Denote the environmental dynamics as $p(s'|s,a)$, the model error $\epsilon_m$, the learned dynamics model as $p_m(s'|s,a)$, and the discounted visitation frequencies with the dynamics model as $\rho_{\pi,m}(s)$. Then, \textbf{optimizing $J(\pi)$ is equivalent to minimizing the expected distance between virtual achieved goals and original desired goals}, i.e., $ED_{\pi,m}(g)=\sum_s \rho_{\pi, m}(s) \|\phi(s)-g\|_2^2$. Specifically, the following inequality holds:
    \begin{equation*}
        ED_{\pi,m}(g) \leq -C_g \cdot J(\pi) +  \frac{C_g+\epsilon}{1-\gamma} + \frac{2\gamma \epsilon_m C_g}{(1-\gamma)^2}
    \end{equation*}
\end{theorem}
\begin{proof}
    We first bound the expected distance as:
    \begin{equation}
    \label{eq:ED_upper}
    \begin{aligned}
         ED_{\pi}(g) &= \sum_s \rho_\pi(s) \|\phi(s)-g\|_2^2 \cdot 1[\|\phi(s)-g\|_2^2 \leq \epsilon] + \sum_s \rho_\pi(s) \|\phi(s)-g\|_2^2 \cdot 1[\|\phi(s)-g\|_2^2 > \epsilon] \\
         & \leq \frac{\epsilon}{1-\gamma} + \sum_s \rho_\pi(s) \|\phi(s)-g\|_2^2 \cdot 1[\|\phi(s)-g\|_2^2> \epsilon] \\
         & \leq \frac{\epsilon}{1-\gamma} + C_g \sum_s \rho_\pi(s) \cdot 1[\|\phi(s)-g\|_2^2> \epsilon]
    \end{aligned}
    \end{equation}
    The above derivation leverages the property that:
    \begin{equation*}
        \begin{aligned}
            \sum_s \rho_{\pi}(s)=\sum_s P(s_0=s) + \gamma \sum_s P(s_1=s) 
            + \gamma^2 \sum_s P(s_2=s)+\dots = \frac{1}{1-\gamma}
        \end{aligned}
    \end{equation*}
    As in our setting, rewards only depend on states, thus the RL objective can be written as:
    \begin{equation*}
        \begin{aligned}
             J(\pi) &= E_{s_0\sim p(s_0),a_t\sim \pi(a_t|s_t),s_{t+1}\sim p(s_{t+1}|s_t,a_t)}\big[\sum_{t=0}^{\infty} \gamma^t r(s_t)\big] \\
            &=\sum_{t=0}^{\infty}\sum_s P(s_t=s|\pi) \cdot \gamma^t r(s) 
            = \sum_s \sum_{t=0}^{\infty} \gamma^t P(s_t=s|\pi) \cdot  r(s) \\
            &= \sum_s \rho_\pi(s) \cdot 1[\|\phi(s)-g\|_2^2 \leq \epsilon] \\
            &= \sum_s \rho_\pi(s) \cdot \big[ 1 - 1[\|\phi(s)-g\|_2^2 > \epsilon] \big] \\
            &= \frac{1}{1-\gamma} - \sum_s \rho_\pi(s) \cdot 1[\|\phi(s)-g\|_2^2 > \epsilon]
        \end{aligned}
    \end{equation*}
    Taking the upper bound of $ED_\pi(g)$ in Eq \ref{eq:ED_upper} into $J(\pi)$, we have:
    \begin{equation*}
        \begin{aligned}
            J(\pi) \leq \frac{1}{1-\gamma} - \frac{1}{C_g} [ED_\pi (g) - \frac{\epsilon}{1-\gamma}]
        \end{aligned}
    \end{equation*}
    and 
    \begin{equation*}
        \begin{aligned}
            ED_\pi (g)  \leq  -C_g \cdot J(\pi) +  \frac{C_g+\epsilon}{1-\gamma} .
        \end{aligned}
    \end{equation*}
    The virtual achieved goals are generated with the learned dynamics model, therefore the following inequality also holds for the learned dynamics model:
    \begin{equation*}
        \begin{aligned}
            ED_{\pi,m}(g)  \leq -C_g \cdot J_m(\pi) +  \frac{C_g+\epsilon}{1-\gamma} .
        \end{aligned}
    \end{equation*}
    From the proof of Lemma \ref{lemma:J}, $J^{(\infty)}_m(\pi)$ can be bounded as: $J_m(\pi) \geq J(\pi) - \frac{2\gamma \epsilon_m}{(1-\gamma)^2}$, where $\epsilon_m$ is the model error defined before.
    Finally, we can conclude that:
    \begin{equation*}
        \begin{aligned}
            ED_{\pi,m}(g) &\leq -C_g \cdot J_m(\pi) +  \frac{C_g+\epsilon}{1-\gamma}  \\
            &\leq -C_g \cdot J(\pi) +  \frac{C_g+\epsilon}{1-\gamma} + \frac{2\gamma \epsilon_m C_g}{(1-\gamma)^2} .
        \end{aligned}
    \end{equation*}
    This completes the proof.
\end{proof}

\begin{theorem}[Expected Distance Bound, Finite Horizon]
    Assume the finite horizon for model-based rollout equals the horizon in real environment, i.e., $n=T$. Following Theorem \ref{tm:expected_distance} and [Ross et al. 2011], the expected distance bound for finite horizon is:
     \begin{equation*}
        ED_{\pi,m}(g) \leq -C_g \cdot J(\pi) +  (C_g+\epsilon)T + 2\gamma \epsilon_m C_g T^2
    \end{equation*}
\end{theorem}

\section{Extra Experimental Results}
\label{ap:extra_experiment}
In this section, we provide additional experimental results to further understand MHER and its effectiveness. We report the average success rate and standard deviation across 10 random seeds for all the experiments.

\begin{figure}[htb]
    \centering
    \includegraphics[width=1\linewidth]{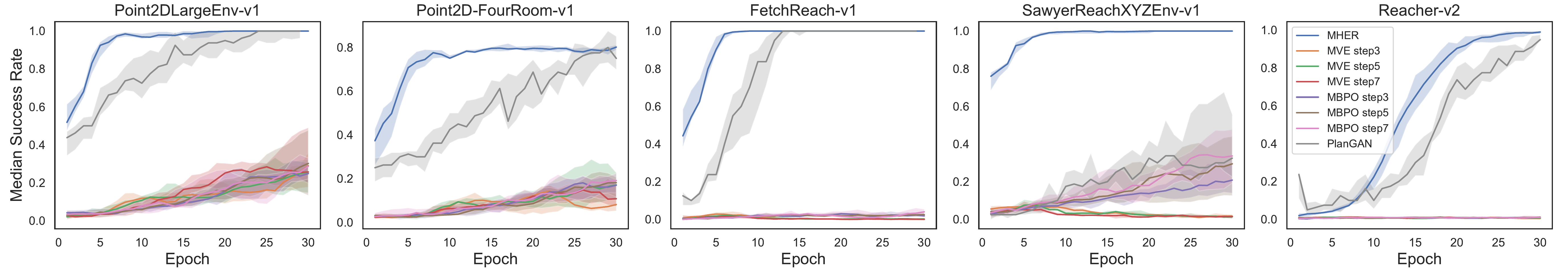}
    \caption{Comparison results with model-based baselines, including PlanGAN.}
    \label{fig:model_bl_planGAN}
\end{figure}

\subsection{Comparison Results with Model-based Baselines}
\label{ap:comparison_with_model_based}
In our paper, we have compared with model-based RL baseines such as MVE \cite{feinberg2018model} and MBPO \cite{DBLP:conf/nips/JannerFZL19}. Furthermore, we include PlanGAN as our baseline, which is not a RL method but a planning method requiring huge amount of computation. The procedure of selection one single action of planner in PlanGAN is: 1) first random sampling 20 initial actions, 2) further sampling 50 trajectories with model and GANs for each initial action, 3) computing average return of each initial action, and 4) selecting the best action in the initial set. Even if we ignore the additional computation to train GANs, PlanGAN needs at least $10^5$ (20 initial action$\times$ 50 trajectories $\times $50 steps $\times 2$ policy and model propagation) times forward propagations than MHER to select a single action. As PlanGAN is highly dependent on the learned model, we train the model and GANs in PlanGAN with a batch size of 128 and each epoch train 25 times. In MHER, we only update the model 2 times each epoch with a batch size of 64, therefore MHER is more computationally friendly compared to PlanGAN. We report the median test success rate and interquartile across 10 random seeds in Figure \ref{fig:model_bl_planGAN}, which apparently shows MHER is more sample efficient than PlanGAN. Moreover, MHER is much more stable compared with PlanGAN and other model-based baselines considering the deviation in the results.

Moreover, we provide comparison results with 
MBPO+HER and MVE+HER in Figure \ref{fig:mbpo_her}. In the two methods, we perform hindsight relabeling before model-based interaction, and then the model-based interaction is also driven by hindsight goals. Through combining with HER, MBPO and MVE work better in sparse reward tasks. However, their performance still cannot surpasses MHER's.

\begin{figure}
    \centering
    \includegraphics[width=1\linewidth]{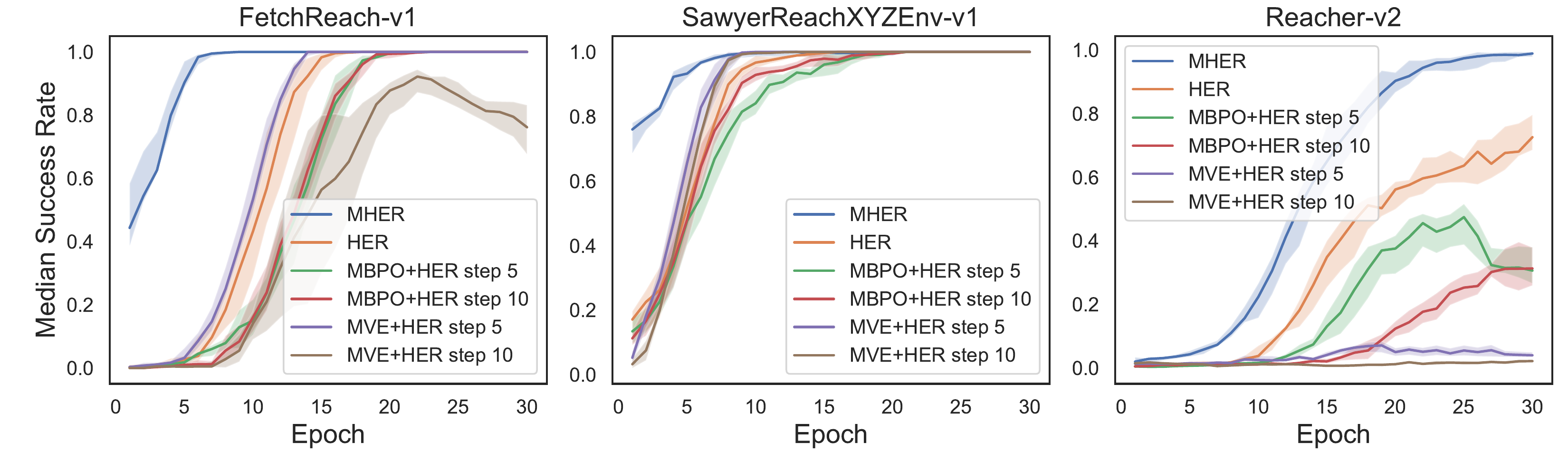}
    \caption{Comparison results with MBPO+HER and MVE+HER in FetchReach-v1, SawyerReachXYZEnv-v1, and Reacher-v2.}
    \label{fig:mbpo_her}
\end{figure}

\subsection{Study of the Model Layers}
In this section, we study the impact of model layers in MHER by varying the number of layers (each of which is of 256 neurons). The empirical results in Figure \ref{fig:layers}  show the number of layers does not affect the performance much and even one layer can perform comparably to 4 layers. The dynamics model with 4 layers achieves stable and good performance in different tasks.

\begin{figure}[htb]
    \centering
    \includegraphics[width=1\linewidth]{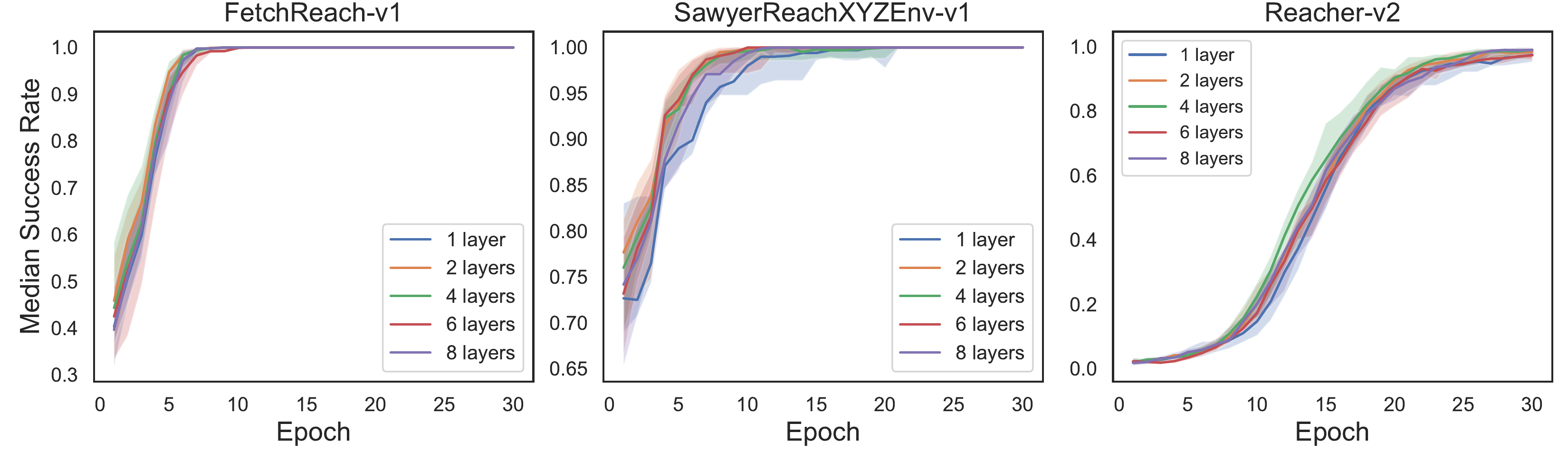}
    \caption{Varying the number of model layers in FetchReach-v1, SawyerReachXYZEnv-v1, and Reacher-v2.}
    \label{fig:layers}
\end{figure}

\subsection{Parameter Study}
Additional parameter studies of $\alpha$ and model-based interaction steps are shown in Figure \ref{fig:parameters_full}. In the results, we can conclude that the performance of MHER improves when parameter $\alpha$ increases from 0 to 3 in the two environments, but decrease slightly when $\alpha$ is beyond 3 in Reacher-v2. Regarding the model-based interaction steps, we can observe a performance increase as steps increase from 0 to 5, and a performance decrease as steps are more than 5.
\begin{figure}[htb]
    \centering
    \includegraphics[width=1\linewidth]{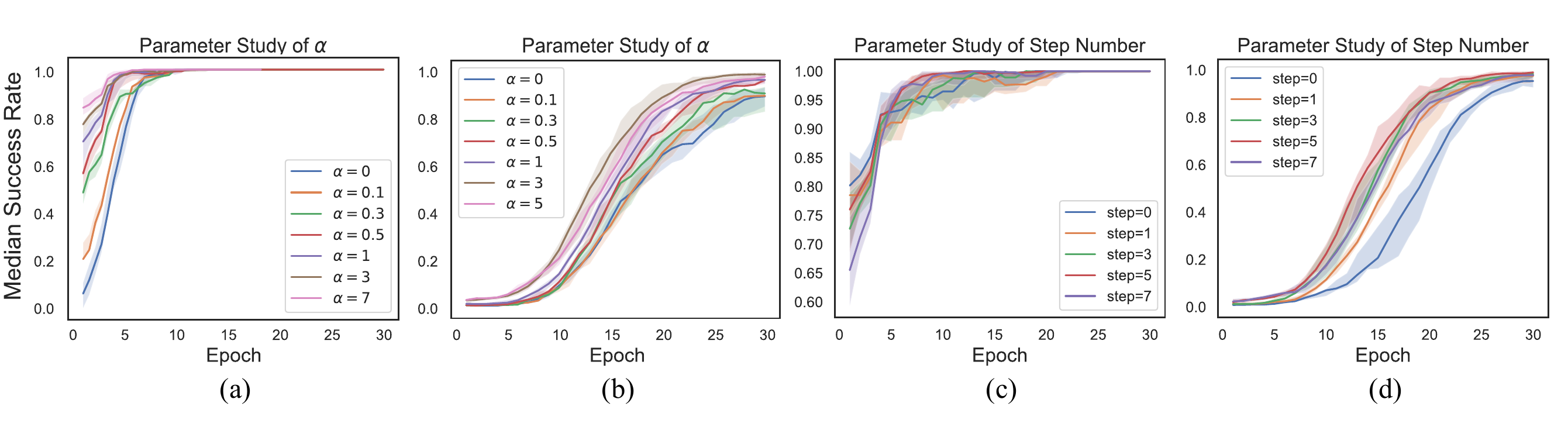}
    \caption{(a)(b) Additional parameter studies of $\alpha$ in SawyerReachXYZEnv-v1 and Reacher-v2. (c)(d) Additional parameter studies of model-based interaction steps in SawyerReachXYZEnv-v1 and Reacher-v2.}
    \label{fig:parameters_full}
\end{figure}

\begin{figure}[htb]
    \centering
    \includegraphics[width=1\linewidth]{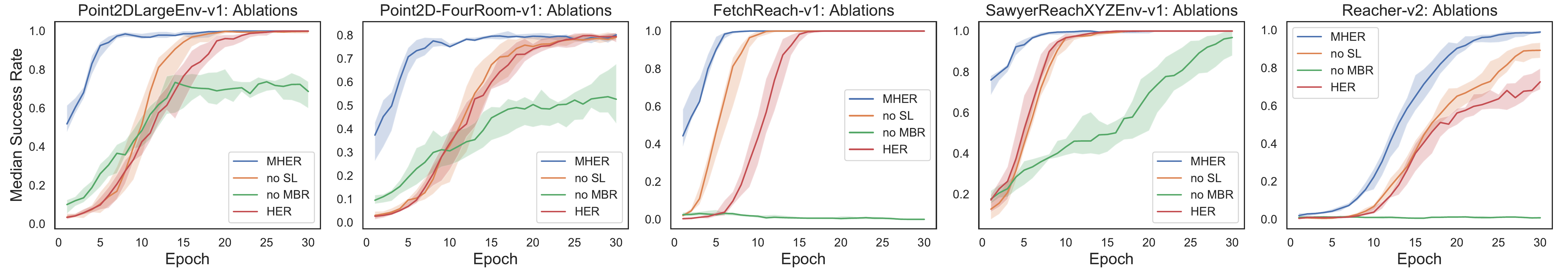}
    \caption{Additional ablation studies in benchmark environments.}
    \label{fig:ablations_full}
\end{figure}

\subsection{Ablations}
We also provide more ablation study results in Figure \ref{fig:ablations_full}. The results are consistent with the conclusions in our paper: (1) MBR is more important than SL, (2) we can outperform HER using MBR alone, and (3) SL is not robust and fails to learn in FetchReach-v1 and Reacher-v2.

\subsection{Additional Comparisons}
We conducted additional experiments in Figure \ref{fig:random_relabeling} to study how does MHER compare to \emph{random relabeling} and HER with \emph{goal noise}. For \emph{random relabeling}, we relabel the transitions with random goals sampled from the goal space. Regarding \emph{goal noise}, Gaussian noise with zero mean and constant (0.01) standard deviation is applied to hindsight goals of HER. We can observe from the results that \emph{goal noise} provides slight improvement over HER in FetchReach-v1 and SawyerReachXYZEnv-v1, but it worsens the result in Reacher-v2. In addition, it seems that \emph{random relabeling} does not work, which might be because it hardly helps with the sparse reward problem.


\begin{figure}[h]
    \centering
    \includegraphics[width=1\linewidth]{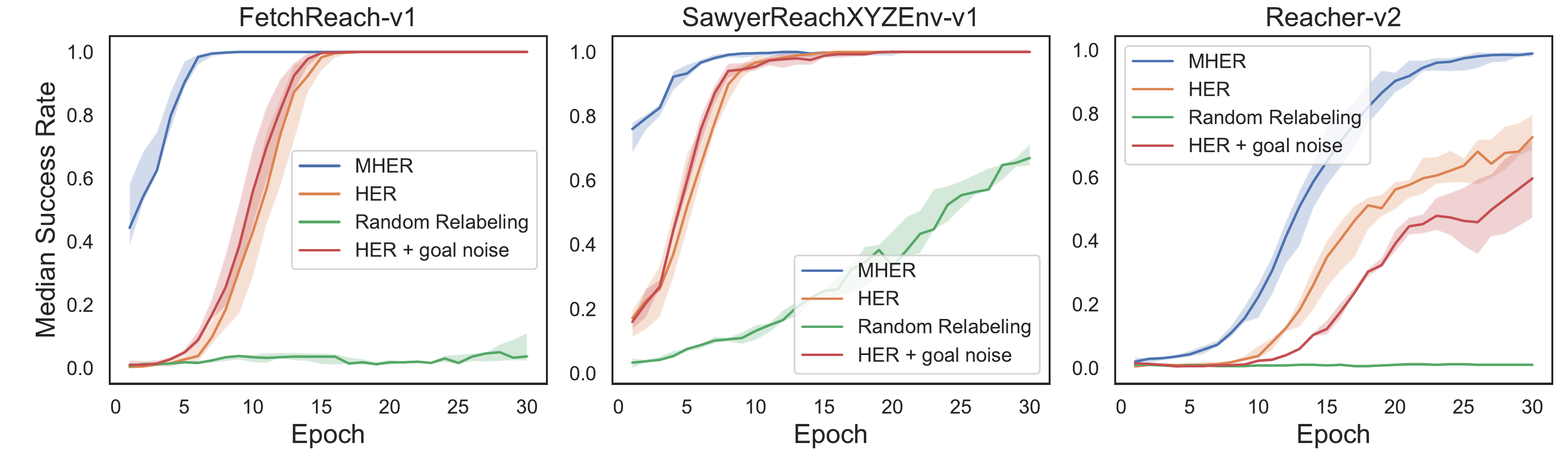}
    \caption{Comparison results between MHER, \emph{random relabeling} and \emph{goal noise}.}
    \label{fig:random_relabeling}
\end{figure}

\section{Details about Virtual Achieved Goals}
\label{ap:virtual_ags}
Previous work, HER \cite{andrychowicz2017hindsight}, assumes there exists a state-to-goal mapping function: $\phi:S\rightarrow G$, i.e., given a state $s$, we can find a goal $g=\phi(s) \in G$ achieved by this state. In the case where each goal corresponds to a state we want to achieve (e.g., 2D point reaching task), goal space equals state space $G=S$ and the mapping $\phi$ is exactly an identity transformation.

Model-based relabeling is a novel goal relabeling method different from previous relabeling methods and enjoys many advantages by leveraging the dynamics model. MHER follows the same setting as HER. The difference is that MHER leverages the virtual states generated from model-based interaction rather than past collected states. As shown in Figure \ref{fig:diagram_ap}, the blue trajectory is collected by past policy, while the green trajectory is generated by interaction of current policy and the learned dynamics model. 

Given a transition in a past collected trajectory $(s_t, a_t, r_t, s_{t+1}, g)$, model-based relabeling (MBR) aims to find a virtual achieved goal to replace the original goal $g$ and reward $r_t$. MBR interacts with the dynamics model $m$ for $n$ steps starting from $s_{t+1}$, and collects a virtual trajectory $\{s'_{t+i}, a'_{t+i}, s'_{t+i+1}\}_{i=1}^{n}$, where $s'_{t+1}=s_{t+1}, a'_{t+i}=\pi(s'_{t+i},g), s'_{t+i+1}=s'_{t+i}+m(s'_{t+i},a'_{t+i}), i\in[1,n]$. Note that we start interaction from $s_{t+1}$, which is reasonable because starting from $s_t$ means $a_t, s_{t+1}$ also need to be replaced. After model-based interaction, MBR samples from virtual achieved goals $g'=\phi(s_{t+i}), i\in[1,n]$ and alternates the original transition as $(s_t,a_t, r'_t, s_{t+1},g')$, where $r'_t$ is the recomputed reward $r'_t=r(s_t,a_t,g')$ according to Eq. \ref{equ:rewardfunction}. The model-based relabeled data can be further used for off-policy reinforcement learning and goal-conditioned supervised learning.
\begin{figure}[htb]
    \centering
    \includegraphics[width=0.85\linewidth, trim=0 0 0 130, clip ]{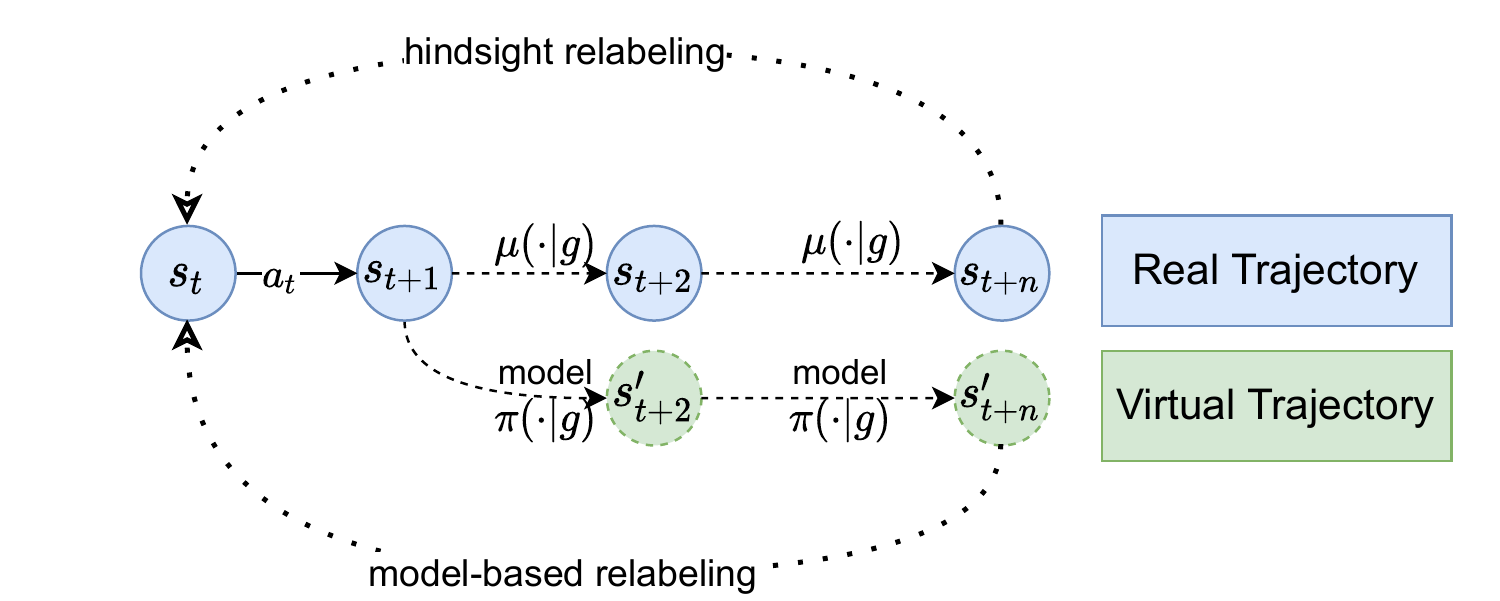}
    \caption{Simplified schematic of model-based relabeling.}
    \label{fig:diagram_ap}
\end{figure}

\section{Implementation Details}
\label{ap:implementation_detail}
\subsection{Implementation of MHER}
In this section, we will provide implementation details of MHER. The actor and critic networks of MHER are both 3-layer fully connected networks with 256 units each and ReLU non-linearities. The actor and critic are updated with Adam optimizer, learning rate $1\times 10^{-3}$. Target nets of actor and critic with the same network structure are adopted for stabilizing training, the Polyak averaging coefficient of the target net is set as $0.9$. The probability of random action is $0.3$, and the scale of Gaussian noisy for exploration is $0.2$. The model-based relabeling rate is $80\%$, and we keep $20\%$ samples without relabeling. MHER uses one rollout worker, and the buffer size is $10^6$. At the end of each episode (including 100 environment steps), we train the networks with $5$ batches of size $64$. We only use one single CPU and one GPU to train MHER. 


Regarding the dynamics model in MHER, we use a fully connected network with $4$ hidden layers and $256$ neurons each layer. The optimizer is Adam and the learning rate is $0.001$. The dynamics model is trained to minimize the loss $\mathcal{L}_{model}$ in Eq. \ref{eq:dynamicloss}. In the warmup period, we train the dynamics model for $100$ updates with a batch size of $512$. When training with MHER, every batch we update $m$ with a batch size of $64$ for $2$ times.

\subsection{Implementation of GCSL}
We implement GCSL with a Diagonal Gaussian policy with a constant standard deviation of $0.2$. GCSL only keeps a single policy network (3 layers, 256 neurons each) as the actor in MHER, without target network. The policy network is trained to minimize the loss 
$$\mathcal{L}_{GCSL} = E_{(s_t,a_t,g') \sim B_h} \big[\|a_t - \pi(s_t,g')\|_2^2 \big]$$
where $g'$ is relabeled using future achieved goals like HER \cite{andrychowicz2017hindsight}. The optimizer is Adam and the learning rate is also $1\cdot10^{-3}$. GCSL shares other parameters with MHER for fair comparison, such as replay buffer size and batch size.

\subsection{Implementation of MVE and MBPO}
Model-based value expansion (MVE) \cite{feinberg2018model} rollouts with a learned dynamics model and introduces multi-step value estimation based on the expanded transitions. Denote the replay buffer as $B$, a real transition $(s_t,a_t,r_t,s_{t+1},g)$, generated transitions $\{\hat s_{t+i}, \hat a_{t+i}, \hat r_{t+i}\}, i\in[1,H], \hat s_{t+1}=s_{t+1}$, MVE updates the Q-function to minimize the following loss:
\begin{equation}
\mathcal{L}_{critic}=E_{(s_t,a_t,g, r_t,s_{t+1})\sim B} \big[( \sum_{i=0}^{H-1} \gamma^i \hat r_{t+i} + \gamma^H Q (\hat s_{t+H}, \pi (\hat s_{t+H}, g), g)-Q(s_t,a_t, g))^2 \big],
\end{equation}
where $\hat r_t=r_t, \hat r_{t+i}=r(\hat s_{t+i}, \pi(\hat s_{t+i}, g), g)$, $r$ is the reward function in Eq. \ref{equ:rewardfunction}. We also use DDPG to learn the policy, and the hyper-parameters of DDPG and dynamics model are the same as MHER. We provide empirical results with varying $H$ in Figure \ref{fig:model_bl_planGAN}.

Model-based policy optimization (MBPO) \cite{DBLP:conf/nips/JannerFZL19} stores short model-generated rollouts branched from real data to the model dataset $D_{model}$, and optimizes policy using the model dataset $D_{model}$. For fair comparison, we also utilize DDPG for policy optimization and the size of $D_{model}$ is $1\cdot 10^6$. Hyper-parameters of DDPG and dynamics model are the same as MHER. The results with model-based horizon $\{3,5,7\}$ are reported in Figure \ref{fig:model_bl_planGAN}.

\section{Comparison with Concurrent Work}
We noticed a concurrent work, MapGo \cite{zhu2021mapgo}, which shares similar idea of model-based relabeling (MBR) with us. However, MapGo doesn't exploit the MBR goals deeper. Instead, we further leverage MBR goals for supervised policy learning with theoretical guarantees. Moreover, MapGo utilizes a large number of virtual rollout for policy improvement, and thus introduces model error in both states and actions. In Section \ref{sec:MBR}, we have emphasized that MHER avoids training with fully virtual states, and only the goals in the training data are generated by the model. Therefore, our method can be more robust to model errors. Despite the differences, MapGo found the model cannot fit exactly for hard manipulation tasks, which we have also observed in our experiments and we believe it is worth further research.

\section{Task Descriptions}
\label{ap:task_description}
All of the five tasks have continuous state space, action space and goal space. In this section, we will introduce these tasks in detail.
\subsection{Point2DLargeEnv-v1}
\begin{wrapfigure}{r}{2.3cm}
\centering
  \vspace{-25pt}    
  \includegraphics[height=1.8cm,width=1.8cm]{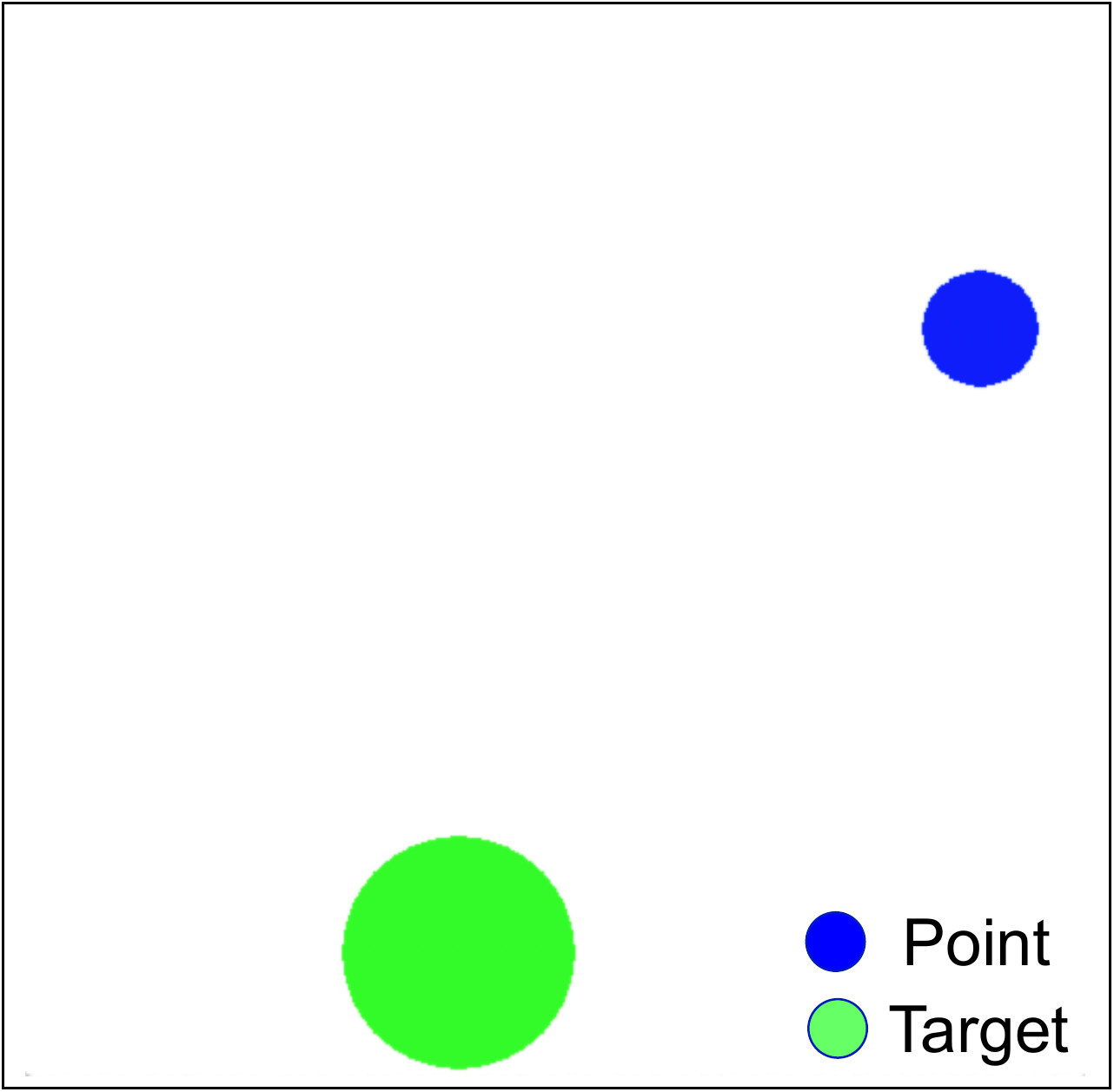}\\
  \vspace{-0pt}    
  \caption*{Point2DLarge}
  \vspace{-8pt}    
\end{wrapfigure}
Point2DLargeEnv-v1 is taken from the open-source package \emph{multi-world} and it requires the blue point to reach the green circle. The state space $[-5,5]\times[-5,5]$ has two dimensions representing Cartesian coordinates of the blue point, and the action space $[-1,1]\times[-1,1]$ also has two dimensions meaning the horizontal and vertical displacement. The goal space is the same as state space, which means $\phi(s)=s$. The bule point and the green circle are randomly initialized in the state space. The allowable error $\epsilon$ of reaching goal is the radius of the target circle and is set as $1$. The reward function is defined as:
\begin{equation*}
    r(s_{XY},a,g_{XY}) = - 1(\|s_{XY}-g_{XY} \|_2^2 > \epsilon) .
\end{equation*}

\subsection{Point2D-FourRoom-v1}
\begin{wrapfigure}{r}{2.3cm}
\centering
  \vspace{-40pt}    
  \includegraphics[height=1.8cm,width=1.8cm]{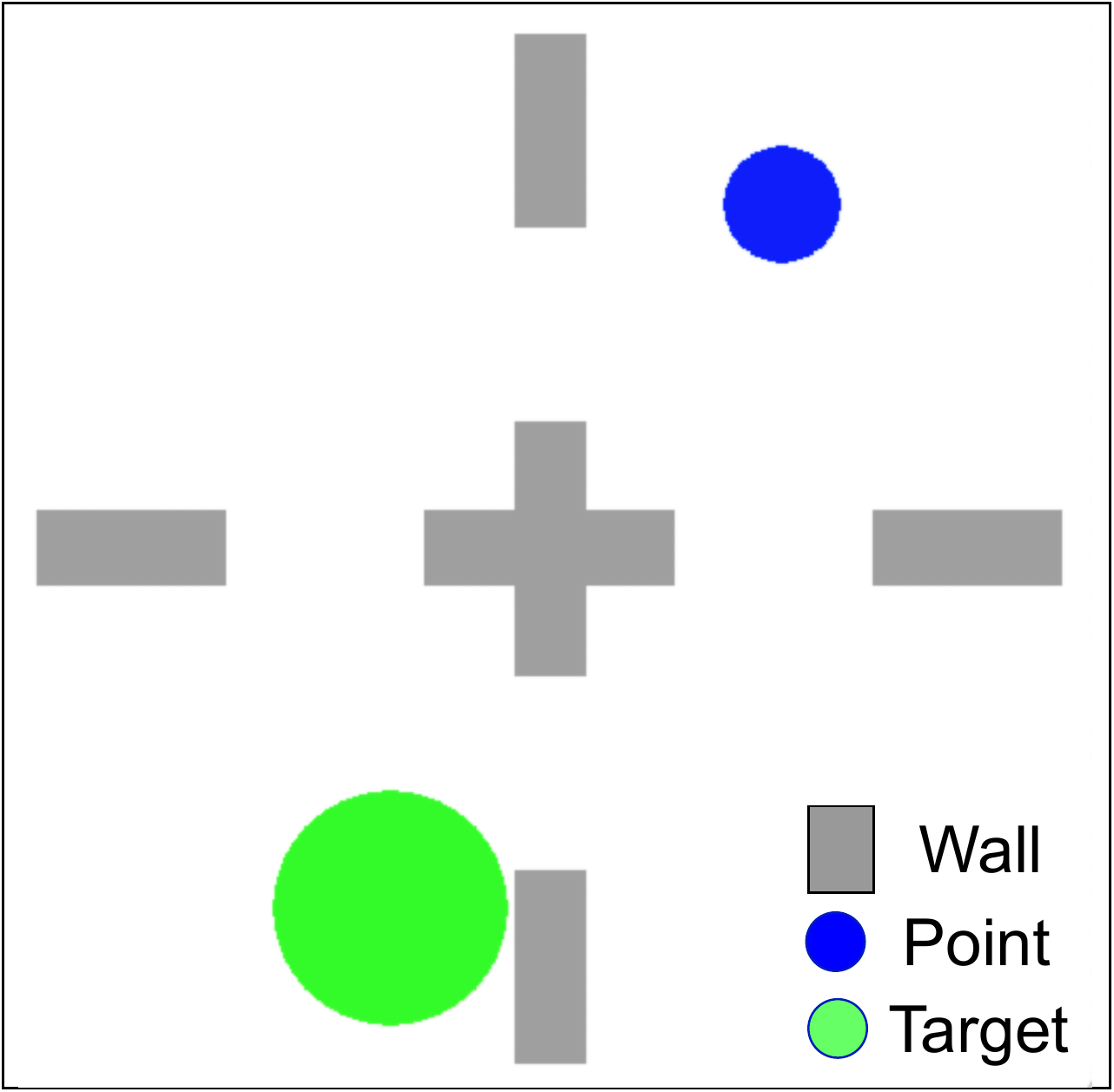}\\
  \vspace{-0pt}    
  \caption*{PointFourRoom}
  \vspace{-0pt}    
\end{wrapfigure}
The Point2D-FourRoom-v1 environment is also built on \emph{multi-world}. The state space, the action space, the goal space, and the reward function are same as Point2DLarge-v1. The difference is that there are four rooms separated by gray walls.

\subsection{FetchReach-v1}
\begin{wrapfigure}{r}{2.3cm}
\centering
  \vspace{-10pt}    
  \includegraphics[height=1.8cm,width=1.8cm]{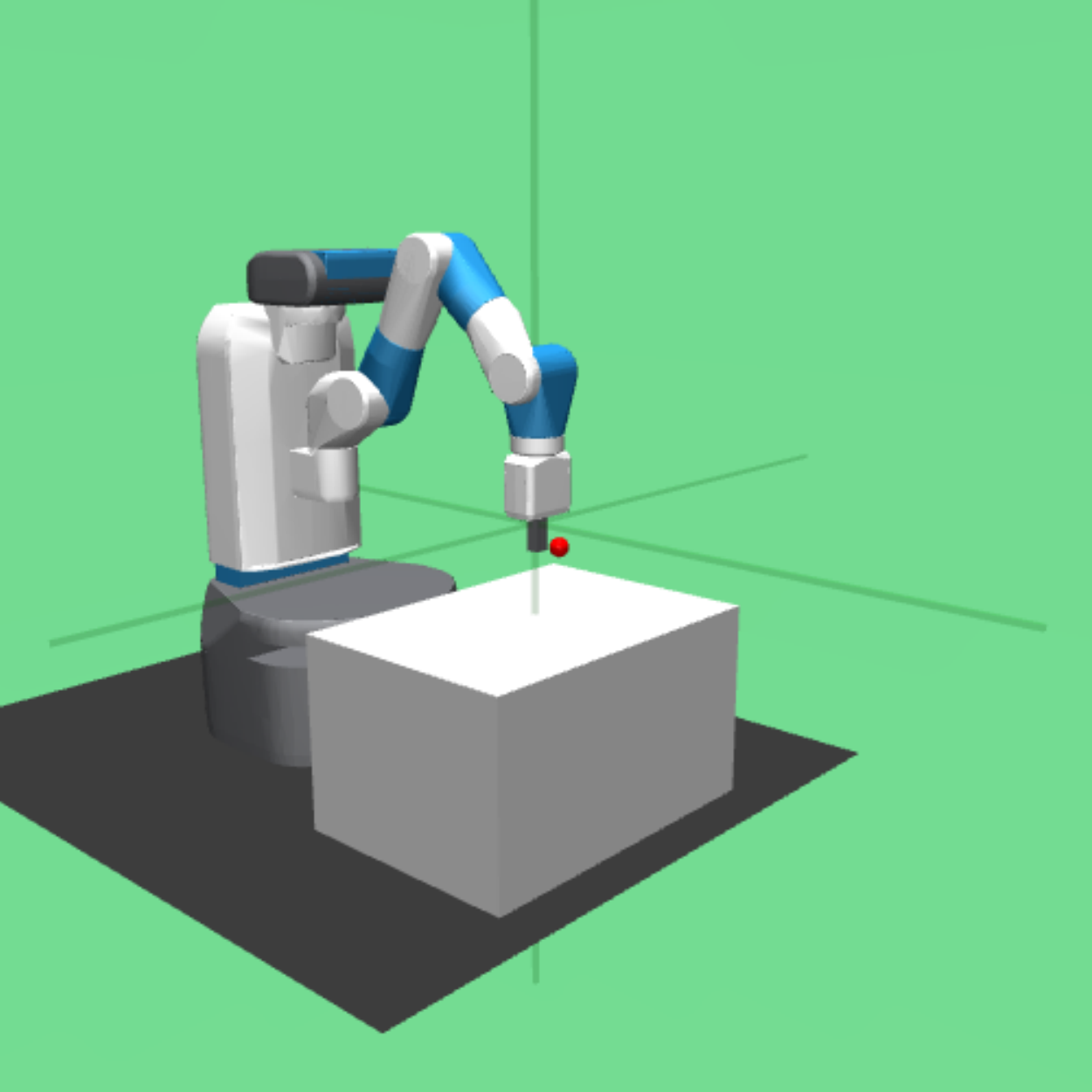}\\
  \vspace{-0pt}    
  \caption*{FetchReach}
  \vspace{0pt}    
\end{wrapfigure}
The FetchReach-v1 environment is taken from OpenAI Gym \cite{brockman2016openai}. In this environment, a 7-DoF robotic arm is expected to touch a desired location with its two-finger gripper. The state space is 10-dimensional, including the gripper’s position and linear velocities. The action space is 4-dimensional, which represents the gripper’s movements and its status about the opening and closing. Moreover, the goals are 3-dimensional vectors representing the target place of the gripper. The state-to-goal mapping is $\phi(s)=s[0:3]$, because the first $3$ dimensions of state describe the position of the gripper. The allowable error in FetchReach is $\epsilon=0.05$. The reward function is defined as:
\begin{equation*}
    r(s_{XYZ},a,g_{XYZ}) = - 1(\|s_{XYZ}-g_{XYZ} \|_2^2 > \epsilon) .
\end{equation*}

\subsection{SawyerReachXYZEnv-v1}
\begin{wrapfigure}{r}{2.3cm}
\centering
  \vspace{-10pt}    
  \includegraphics[height=1.8cm,width=1.8cm]{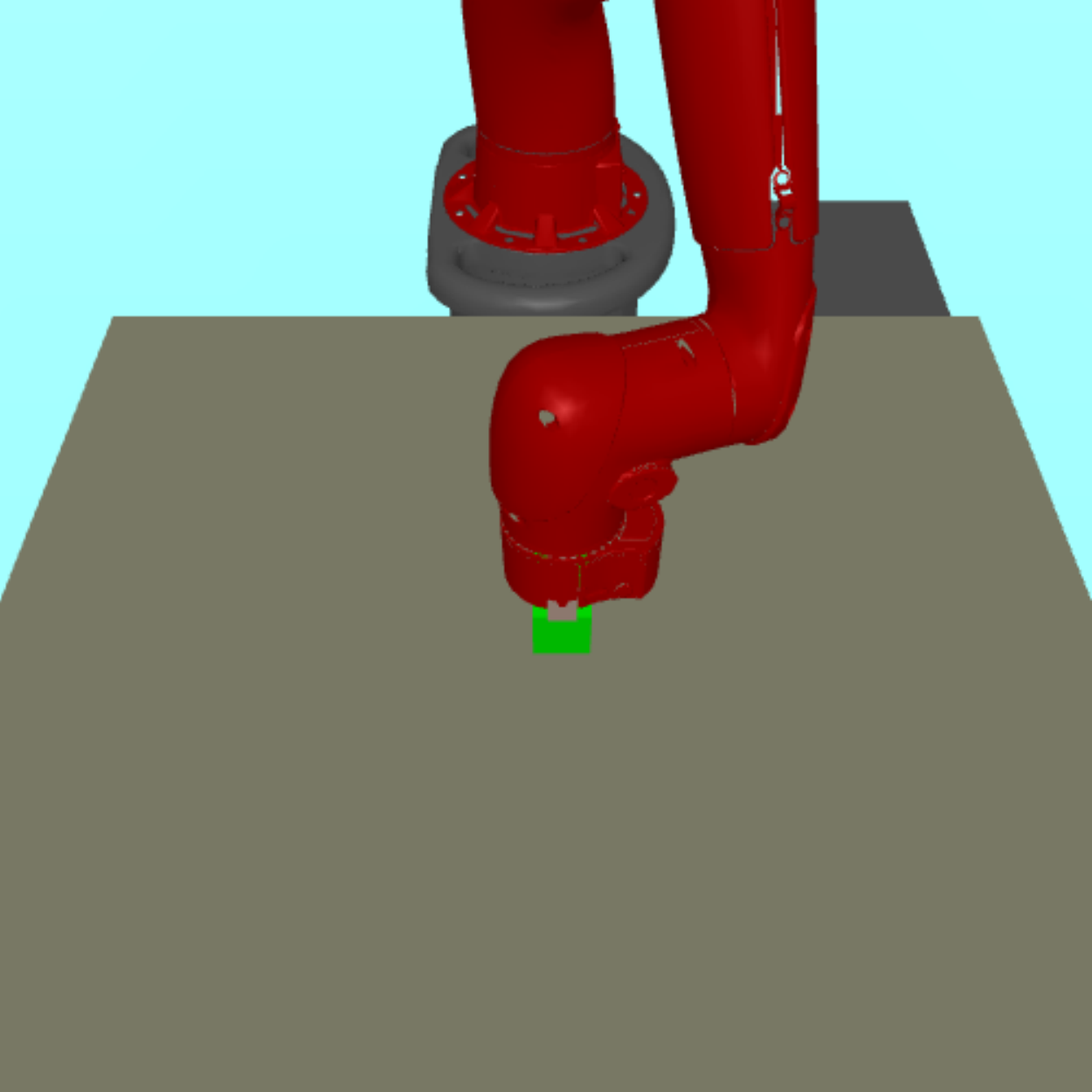}\\
  \vspace{-0pt}    
  \caption*{SawyerReach}
  \vspace{-8pt}    
\end{wrapfigure}
The SawyerReachXYZEnv environment is taken from \emph{multi-world}. The Sawyer robot aims to reach a target position with its end-effector. The observation space is 3-dimensional, representing the 3D Cartesian position of the end-effector. Correspondingly, the goal space is 3-dimensional and describes the expected position, and the state-to-goal mapping is $\phi(s)=s$. Besides, the action space has 3 dimensions describing the next position of the end-effector. The reward function is the same as FetchReach except the allowable error $\epsilon=0.06$.

\subsection{Reacher-v2}
\begin{wrapfigure}{r}{2.3cm}
\centering
  \vspace{-10pt}    
  \includegraphics[height=1.8cm,width=1.8cm]{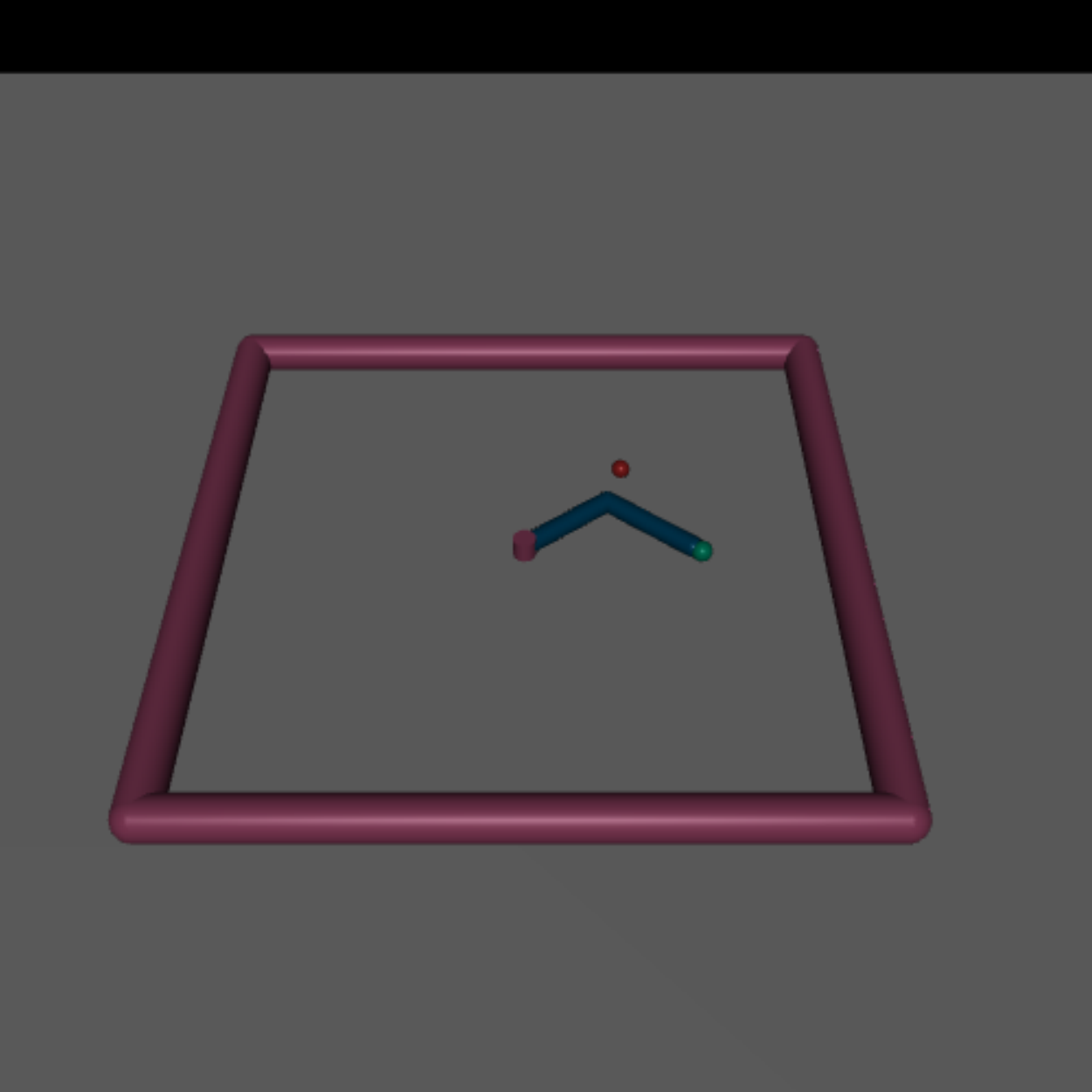}\\
  \vspace{-0pt}    
  \caption*{Reacher}
  \vspace{-8pt}    
\end{wrapfigure}
The Reacher environment is revised from OpenAI Gym \cite{brockman2016openai}. States are 11-dimensional, which indicates the angles, the positions, and the velocity of the joints. Actions are $2$-dimensional and control the movement of two joints. The goals are $2$-dimension representing the expected XY position. And the state-to-goal mapping is $\phi(s)=s[-3:-1]$, where the last three dimensions are the XYZ position of the end-effect. The reward function is defined as:
\begin{equation*}
    r(s_{XY},a,g_{XY}) = - 1(\|s_{XY}-g_{XY} \|_2^2 > \epsilon),
\end{equation*}
where the allowable error $\epsilon$ is set as $0.02$.

\end{document}